\documentclass[accepted]{article}

\usepackage[T1]{fontenc}
\usepackage[utf8]{inputenc}
\usepackage[english]{babel}

\usepackage{newtxtext}

\usepackage{amsmath}
\usepackage{amsthm}

\usepackage[nonewtxmathopt]{newtxmath} 
\usepackage[cal=cm,bb=ams,frak=pxtx]{mathalpha}

\PassOptionsToPackage{hyphens}{url}
\usepackage[hyphens]{url}
\usepackage{hyperref}

\usepackage{graphicx}
\usepackage{caption}
\usepackage{tikz}
\usepackage{pgfplots}
\pgfplotsset{compat=1.17}

\usepackage{booktabs}
\usepackage{microtype}
\usepackage{xcolor}
\usepackage{enumitem}
\setitemize{itemsep=0pt,parsep=0pt,topsep=0pt}
\setenumerate{itemsep=0pt,parsep=0pt,topsep=0pt}

\usepackage{natbib}

\usepackage{style/icml2021}

\newcommand{\card}[1]{\# #1 }
\newcommand{\diff}{\mathop{}\!\mathrm{d}}
\newcommand{\ind}[1]{\mathbf{1}_{#1}}
\newcommand{\prob}[1]{\Delta_{#1}}
\DeclareMathOperator{\E}{\mathbb{E}}
\DeclareMathOperator{\Pbb}{\mathbb{P}}
\DeclareMathOperator*{\argmax}{arg\,max}
\DeclareMathOperator*{\argmin}{arg\,min}

\DeclareMathOperator{\Span}{Span}
\DeclareMathOperator{\hull}{Conv}
\DeclareMathOperator{\sign}{sign}
\DeclareMathOperator*{\supp}{supp}
\DeclareMathOperator{\trace}{Tr}

\newcommand{\N}{\mathbb{N}}

\newcommand{\R}{\mathbb{R}}

\newcommand{\bbracket}[1]{\left\llbracket #1 \right\rrbracket}
\renewcommand{\brace}[1]{\left\{ #1 \right\}}
\newcommand{\bracket}[1]{\left[ #1 \right]}

\newcommand{\paren}[1]{\left( #1 \right)}
\newcommand{\midvert}{\,\middle\vert\,}

\newcommand{\abs}[1]{\left| #1 \right|}
\newcommand{\norm}[1]{\left\| #1 \right\|}
\newcommand{\scap}[2]{\left\langle #1, #2 \right\rangle}

\newcommand{\Sfrak}{\mathfrak{S}}
\newcommand{\X}{\mathcal{X}}
\newcommand{\Y}{\mathcal{Y}}

\renewcommand{\epsilon}{\varepsilon}
\renewcommand{\phi}{\varphi}

\theoremstyle{plain}
\newtheorem{theorem}{Theorem}

\newtheorem{lemma}[theorem]{Lemma}
\newtheorem{proposition}[theorem]{Proposition}
\newtheorem{definition}[theorem]{Definition}
\newtheorem{assumption}{Assumption}
\newtheorem{remark}[theorem]{Remark}

\icmltitlerunning{Disambiguation of weak supervision with exponential convergence rates}

\begin{document}

\twocolumn[
\icmltitle{Disambiguation of Weak Supervision leading to Exponential Convergence rates}

\begin{icmlauthorlist}
\icmlauthor{Vivien Cabannes}{to}
\icmlauthor{Francis Bach}{to}
\icmlauthor{Alessandro Rudi}{to}
\end{icmlauthorlist}

\icmlaffiliation{to}{Institut National de Recherche en Informatique et en
  Automatique -- D\'epartement d'Informatique de l'\'Ecole Normale Sup\'erieure
  -- PSL Research University} 

\icmlcorrespondingauthor{Vivien Cabannes}{vivien.cabannes@gmail.com}

\icmlkeywords{Machine learning, ICML, weak supervision, partial labelling, superset learning, semi-supervised learning, ranking with partial ordering, disambiguation, data annotation, missing data, structured prediction, exponential convergences rates, generalization bonds, cluster assumptions, low-density separation, non-ambiguity, nearest neighbors, state-of-the-art}

\vskip 0.3in
]

\printAffiliationsAndNotice{}  

\begin{abstract}
  Machine learning approached through supervised learning requires expensive
  annotation of data. This motivates weakly supervised learning, where data are
  annotated with incomplete yet discriminative information. In this paper, we
  focus on partial labelling, an instance of weak supervision where, from a given
  input, we are given a set of potential targets. We review a disambiguation
  principle to recover full supervision from weak supervision, and propose an
  empirical disambiguation algorithm. We prove exponential convergence rates of
  our algorithm under classical learnability assumptions, and we illustrate the
  usefulness of our method on practical examples.
\end{abstract}

\section{Introduction}

In many applications of machine learning, such as recommender systems, where an
input $x$ characterizing a user should be matched with a target $y$ representing
an ordering of a large number $m$ of items, accessing fully supervised data
$(x, y)$ is not an option.
Instead, one should expect weak information on the target $y$, which could be a
list of previously taken (if items are online courses), watched
(if items are plays), {\em etc.}, items by a user characterized by the feature vector $x$.
This motivates {\em weakly supervised learning}, aiming at learning a mapping
from inputs to targets in such a setting where tools from supervised learning
can not be applied off-the-shelves.

Recent applications of weakly supervised learning showcase impressive results in
solving complex tasks such as action retrieval on instructional videos
\citep{Miech2019}, image semantic segmentation \citep{Papandreou2015}, salient
object detection \citep{Wang2017}, 3D pose estimation \citep{Dabral2018},
text-to-speech synthesis \citep{Jia2018}, to name a few.
However, those applications of weakly supervised learning are usually based on clever
heuristics, and theoretical foundations of learning from weakly supervised data
are scarce, especially when compared to statistical learning literature on
supervised learning \citep{Vapnik1995,Boucheron2005,Steinwart2008}. We aim to provide a step
in this direction. 

In this paper, we focus on partial labelling, a popular instance of weak
supervision, approached with a structured prediction point of view \cite{Ciliberto2020}. We detail
this setup in Section \ref{sec:setup}.
Our contributions are organized as follows.
\begin{itemize}
  \item In Section \ref{sec:algo}, we introduce a disambiguation algorithm to
    retrieve fully supervised samples from weakly supervised ones, before applying
    off-the-shelf supervised learning algorithms to the completed dataset.
  \item In Section \ref{sec:stat}, we prove exponential convergence rates of our
    algorithm, in term of the fully supervised excess of risk, given classical
    learnability assumptions.
  \item In Section \ref{sec:optim}, we explain why disambiguation algorithms are
    intrinsically non-convex, and provide guidelines based on well-grounded
    heuristics to implement our algorithm.
\end{itemize}
We end this paper with a review of literature in Section \ref{sec:litterature},
before showcasing the usefulness of our method on practical examples in Section
\ref{sec:example}, and opening on perspectives in Section \ref{sec:opening}.

\section{Disambiguation of Partial Labelling}
\label{sec:setup}

In this section, we review the supervised learning setup, introduce the partial
labelling problem along with a principle to tackle this instance of weak supervision.

Algorithms can be formalized as mapping an input $x$ to a desired output $y$,
respectively belonging to an input space~$\X$ and an output space $\Y$.
Machine learning consists in automating the design of the mapping $f:\X\to\Y$,
based on a joint distribution $\mu \in \prob{\X\times\Y}$ over input/output
pairings $(x, y)$ and a loss function $\ell:\Y\times\Y\to\R$,
measuring the error cost of outputting $f(x)$ when one should have output $y$. 
The optimal mapping is defined as satisfying
\begin{equation}
  \label{eq:sol_fs}
  f^* \in \argmin_{f:\X\to \Y} \E_{(X, Y) \sim \mu} \bracket{\ell(f(X), Y)}.
\end{equation}
In \emph{supervised learning}, it is assumed that one does not have access to the full
distribution $\mu$, but only to independent samples
$(X_i, Y_i)_{i\leq n} \sim \mu^{\otimes n}$.
In practice, accessing such samples means building a dataset of examples. While
input data $(x_i)$ are usually easily accessible, getting output pairings
$(y_i)$ generally requires careful annotation, which is both time-consuming and expensive.
For example, in image classification, $(x_i)$ can be collected by scrapping
images over the Internet. Subsequently a ``data labeller'' might be ask to
recognize a rare feline $y_i$ on an image $x_i$. While getting $y_i$ will be
hard in this setting, recognizing that it is a feline and describing elements of
color and shape is easy, and already helps to determine what outputs $f(x_i)$
are acceptable. A second example is given when pooling a known population
$(x_i)$ to get estimation of their political orientation $(y_i)$, one might get
information from recent election of percentage of voters across the political
landscape, leading to global constraints that $(y_i)$ should verify.
A supervision that gives information on $(y_i)_{i\leq n}$ without giving its
precise value is called \emph{weak supervision}.

\emph{Partial labelling}, also known as ``superset learning'', is an instance of
weak supervision, in which, for an input $x$, we do not access the precise label
$y$ but only a set $s$ of potential labels, $y\in s\subset \Y$. For example, on
a caracal image $x$, one might not get the label ``caracal'' $y$, but the set
$s$ ``feline'', containing all the labels $y$ corresponding to felines.
It is modelled through a distribution $\nu \in \prob{\X\times 2^\Y}$ over
$\X\times 2^\Y$ generating samples $(X, S)$, which should be compatible with the
fully supervised distribution $\mu \in \prob{\X\times\Y}$ as formalized by the
following definition.

\begin{definition}[Compatibility, \citet{Cabannes2020}]
  \label{def:compatibility}
  A fully supervised distribution $\mu \in \prob{\X\times\Y}$ is
  \emph{compatible} with a weakly supervised distribution
  $\nu \in \prob{\X\times\Y}$, denoted by $\mu\vdash\nu$ if there exists an
  underlying distribution 
  $\pi \in \prob{\X\times\Y\times 2^\Y}$, such that $\mu$, and $\nu$, are the
  respective marginal distributions of $\pi$ over $\X\times \Y$ and
  $\X\times 2^\Y$, and such that $y\in s$ for any tuple $(x, y, s)$
  in the support of $\pi$ (or equivalently $\pi\vert_{s} \in \prob{s}$, with
  $\pi\vert_{s}$ denoting the conditional distribution of $\pi$ given $s$).
\end{definition}

This definition means that a weakly supervised sample $(X, S) \sim \nu$ can be
thought as proceeding from a fully supervised sample $(X, Y) \sim \mu$ after
loosing information on $Y$ according to the sampling of $S\sim \pi\vert_{X, Y}$.
The goal of partial labelling is still to learn $f^*$ from
Eq.~\eqref{eq:sol_fs}, yet without accessing a fully supervised distribution
$\mu \in \prob{\X\times\Y}$ but only the weakly supervised distribution
$\nu \in \prob{\X\times 2^\Y}$.
As such, this is an ill-posed problem, since $\nu$ does not discriminate between
all $\mu$ compatible with it.  Following {\em lex parsimoniae},
\citet{Cabannes2020} have suggested to look for $\mu$ such that the labels are the
most deterministic function of the inputs, which they measure with a loss-based
``variance'', leading to the disambiguation
\begin{equation}
  \label{eq:principle}
  \mu^* \in \argmin_{\mu\vdash\nu} \inf_{f:\X\to\Y} \E_{(X, Y)\sim\mu}\bracket{\ell(f(X), Y)},
\end{equation}
and to the definition of the optimal mapping $f^*:\X\to\Y$
\begin{equation}
  \label{eq:solution}
  f^* \in \argmin_{f:\X\to\Y} \E_{(X, Y)\sim\mu^*}\bracket{\ell(f(X), Y)}.
\end{equation}
This principle is motivated by Theorem 1 of \citet{Cabannes2020} showing that
$f^*$ in Eq.~\eqref{eq:solution} is characterized by
$f^* \in \argmin_{f:\X\to\Y} \E_{(X, S)\sim\nu}\bracket{\inf_{y\in S} \ell(f(X), y)}$,
matching a prior formulation based on infimum loss
\citep{Cour2011,Luo2010,Hullermeier2014}.
In practice, it means that if $(S\vert X=x)$ has probability 50\% to be the set
``feline'' and 50\% the set ``orange with black stripes'', $(Y\vert X=x)$ should
be considered as 100\% ``tiger'', rather than 20\% ``cat'', 30\% ``lion'' and
50\% ``orange car with black stripes'', which could also explain $(S\vert X=x)$.
In other terms, Eq. \eqref{eq:principle} creates consensus between the different information provided on a label.
Similarly to supervised learning, partial labelling consists in retrieving $f^*$
without accessing $\nu$ but only samples $(X_i, S_i)_{i\leq n} \sim \nu^{\otimes n}$.

\begin{remark}[Measure of determinism]
  \label{rmk:determinism}
  Eq.~\eqref{eq:principle} is not the only variational way to push towards
  distribution where labels are deterministic function of the inputs. For example, one could
  minimize entropy \citep[{\em e.g.,}][]{Berthelot2019,Lienen2021}. However, a
  loss-based principle is 
  appreciable since the loss usually encodes structures of the output space \citep{Ciliberto2020},
  which will allow sample and computational complexity of consequent algorithms
  to scale with an intrinsic dimension of the space rather than the real one,
  e.g., $m$ rather than $m!$ when $\Y = \Sfrak_m$ and $\ell$ is a suitable
  ranking loss \citep[see Section~\ref{sec:ranking} or ][]{Nowak2019}.
\end{remark}

\section{Learning Algorithm}
\label{sec:algo}

In this section, given weakly supervised samples, we present a disambiguation
algorithm to retrieve fully supervised samples based on an empirical expression
of Eq.~\eqref{eq:principle}, before learning a mapping from $\X$ to $\Y$ based
on those fully supervised samples, according to Eq.~\eqref{eq:solution}.

Given a partially labelled dataset ${\cal D}_n = (x_i, s_i)_{i\leq n}$, sampled
accordingly to $\nu^{\otimes n}$, we retrieve fully supervised samples, based on
the following empirical version of Eq.~\eqref{eq:principle}, with
$C_n = \prod_{i\leq n} s_i \subset \Y^n$
\begin{equation}
  \label{eq:disambiguation}
  (\hat{y}_i)_{i\leq n} \in \argmin_{(y_i)_{i\leq n} \in C_n} \inf_{(z_i)_{i\leq n} \in \Y^n} \sum_{i, j=1}^n
  \alpha_j(x_i) \ell(z_i, y_j),
\end{equation}
where $(\alpha_i(x))_{i\leq n}$ is a set of weights measuring how much one
should base its prediction for $x$ on the observations made at $x_i$.
This formulation is motivated by the Bayes approximate rule proposed by
\citet{Stone1977}, which can be seen as the approximation of $\mu$ by
$n^{-1}\sum_{i,j=1}^n \alpha_j(x_i)\delta_{x_i} \otimes \delta_{y_j}$ in
Eq.~\eqref{eq:principle}.
In essence, $z_i$ (which corresponds to $f(x_i)$) is likely to be $y_i$, although Eq. \eqref{eq:disambiguation} allows for flexibility to avoid rigid interpolation.
As such Eq. \eqref{eq:disambiguation} should be understood as 
constraining $y_i\in s_i$ to be similar to $y_j\in s_j$ if $x_i$ and $x_j$ are close,
with the measure of similarity defined by $\alpha_i(x_j)$.

Once fully supervised samples $(x_i, \hat{y_i})$ have been recollected, one can
learn $f_n:\X\to\Y$, approximating $f^*$, with classical supervised learning
techniques. In this work, we will consider the structured prediction estimator
introduced by 
\citet{Ciliberto2016}, defined as
\begin{equation}
  \label{eq:estimate}
  f_n(x) \in \argmin_{z\in \Y} \sum_{i=1}^n \alpha_i(x) \ell(z, \hat{y}_i).
\end{equation}

\paragraph{Weighting scheme $\alpha$.}
For the weighting scheme $\alpha$, several choices are appealing. Laplacian
diffusion is one of them as it incorporates a prior on low density separation to
boost learning \citep{Zhu2003,Zhou2003,Bengio2006,Hein2007}.
Kernel ridge regression is another due to its theoretical guarantees
\citep{Ciliberto2020}.
In the theoretical analysis, we will use nearest neighbors.
Assuming $\X$ is
endowed with a distance $d$, and assuming, for readability sake, that ties to
define nearest neighbors do not happen, it is defined as
\[
  \alpha_i(x) = \left\{
    \begin{array}{cc}
      k^{-1} & \text{if}\quad \sum_{j=1}^n \ind{d(x, x_j) \leq d(x, x_i)} \leq k\\
      0 & \text{otherwise},
    \end{array}\right.
\]
where $k$ is a parameter fixing the number of neighbors.
Our analysis, leading to Theorem \ref{thm:convergence}, also holds for other
local averaging methods such as partitioning or Nadaraya-Watson estimators. 

\section{Consistency Result}
\label{sec:stat}

In this section, we assume $\Y$ finite, and prove the convergence of $f_n$
towards $f^*$ as $n$, the number of samples, grows to infinity.
To derive such a consistency result, we introduce a  surrogate problem that we
relate to the risk through a calibration inequality. We later assume that
weights are given by nearest neighbors and review classical assumptions, that we
work to derive exponential convergence rates.

In the following, we are interested in bounding the expected generalization
error, defined as
\begin{equation}
  \label{eq:excess}
  {\cal E}(f_n) = \E_{{\cal D}_n} {\cal R}(f_n) - {\cal R}(f^*),
\end{equation}
where
\(
  {\cal R}(f) = \E_{(X, Y)\sim\mu^*}\bracket{\ell(f(X), Y)},
\)
by a quantity that goes to zero, when $n$ goes to infinity.
This implies convergence in probability (the randomness being inherited from 
${\cal D}_n$) of ${\cal R}(f_n)$ towards $\inf_{f:\X\to\Y} {\cal R}(f)$, 
which is referred as \emph{consistency} of the learning algorithm.
We first introduce a few objects.

\paragraph{Disambiguation ground truth \texorpdfstring{$(y_i^*)$}{}.}
Introduce $\pi^* \in \prob{\X\times\Y\times 2^\Y}$ expressing the
compatibility of $\mu^*$ and $\nu$ as in Definition \ref{def:compatibility}.
Given samples $(x_i, s_i)_{i\leq n}$ forming a dataset ${\cal D}_n$, we enrich
this dataset by sampling $y_i^* \sim \pi^*\vert_{x_i, s_i}$, which build an
underlying dataset $(x_i, y_i^*, s_i)$ sampled accordingly $(\pi^*)^{\otimes n}$.
Given ${\cal D}_n$, while {\em a priori}, $y_i^*$ are random
variables, sampled accordingly to $\pi^*\vert_{x_i, s_i}$, because of the
definition of $\mu^*$ \eqref{eq:principle}, under basic
definition assumptions, they are actually deterministic,
defined as $y_i^* = \argmin_{y\in s_i} \ell(f^*(x_i), y)$.
As such, they should be seen as ground truth for $\hat{y}_i$.

\paragraph{Surrogate estimates.}
The approximate Bayes rule was successfully analyzed recently through the prism of
plug-in estimators by \citet{Ciliberto2020}. While we will not cast our
algorithm as a plug-in estimator, we will leverage this surrogate approach,
introducing two mappings $\phi$ and $\psi$ from~$\Y$ to an Hilbert space
${\cal H}$ such that
\begin{equation}
  \label{eq:loss}
  \forall\, z, y \in \Y, \qquad \ell(z, y) = \scap{\psi(z)}{\phi(y)},
\end{equation}
Such mappings always exist when $\Y$ is finite, and have been used to encode
``problem structure'' defined by the loss $\ell$ \citep{Nowak2019}.
Note that many losses ({\em e.g.} Hamming, Spearman, Kendall in ranking) can be written as correlation losses which corresponds to $\psi = -\phi$, yet Eq. \eqref{eq:loss} allows to model much more losses, especially asymmetric losses ({\em e.g.} discounted cumulative gain).
We introduce three surrogate quantities that will play a major role in the
following analysis, they map $\X$ to ${\cal H}$ as
\begin{gather}
  \label{eq:surrogate}
  g^*(x) = \E_{\mu^*}\bracket{\phi(Y)\midvert X=x},\qquad
  g_n(x) = \sum_{i=1}^n \alpha_i(x) \phi(\hat{y}_i),\nonumber\\
  g_n^*(x) = \sum_{i=1}^n \alpha_i(x) \phi(y^*_i).
\end{gather}
It is known that $f^*$ and $f_n$ are retrieved from $g^*$ and $g_n$, through the
decoding, retrieving $f:\X\to\Y$ from $g:\X\to{\cal H}$ as
\begin{equation}
  \label{eq:decoding}
  f(x) = \argmin_{z\in\Y} \scap{\psi(z)}{g(x)},
\end{equation}
which explains the wording of {\em plug-in} estimator \citep{Ciliberto2020}.
We now introduce a {\em calibration inequality}, that relates the error between
$f_n$ and $f^*$ with surrogate error quantities.

\begin{lemma}[Calibration inequality]\label{lem:cal}
  When $\Y$ is finite, and the labels are a deterministic function of the input,
  {\em i.e.}, when   $\mu^*\vert_x$ is a Dirac for all $x\in\supp\nu_\X$, 
  for any weighting scheme that sums to one,
  {\em i.e.}, $\sum_{i=1}^n \alpha_i(x) = 1$ for all $x\in\supp\nu_\X$,  
  \begin{align}
    \label{eq:cal}
    &{\cal R}(f_n) - {\cal R}(f^*) \leq 4c_\psi \norm{g_n^* - g_n}_{L^1} \hspace*{2cm}
    \nonumber \\& \hspace*{1cm} \qquad+
    8c_\psi c_\phi \Pbb_X\paren{\norm{g_n^*(X) - g^*(X)} > \delta},
  \end{align}
  with $c_\psi = \sup_{z\in\Y} \norm{\psi(z)}$,
  $c_\phi = \sup_{z\in\Y}\norm{\phi(y)}$, and
  $\delta$ a parameter that depend on the geometry of $\ell$ and its
  decomposition through $\phi$.
\end{lemma}

This lemma, proven in Appendix \ref{proof:cal}, separates a part reading in
$\norm{g_n - g_n^*}$, due to the {\em disambiguation error} between $(\hat{y}_i)$ and
$(y_i^*)$ together with the {\em stability} of the learning algorithm when
substituting $(\hat{y}_i)$ for $(y_i^*)$, and a part in $\norm{g_n^* - g^*}$
due to the {\em consistency} of the fully supervised learning algorithm. The
expression of the first part relates to Theorem 7 in \citet{Ciliberto2020} while
the second part relates to Theorem 6 in \citet{Cabannes2021}.

\subsection{Classical learnability assumptions}
In the following, we suppose that the weights $\alpha$ are given by nearest
neighbors, that $\X$ is a compact metric space endowed with a distance $d$,
that $\Y$ is finite and that $\ell$ is proper in the sense that it strictly positive except on the diagonal of $\Y\times\Y$ diagonal where it is zero.
We now review classical assumptions to prove consistency.
First, assume that $\nu_\X$ is regular in the following sense.

\begin{assumption}[$\nu_\X$ well-behaved]\label{ass:mass}
  Assume that $\nu_\X$ is such that there exists $h_1, c_\mu, q > 0$ satisfying,
  with ${\cal B}$ designing balls in $\X$, 
  \[
    \forall\, x\in\supp\nu_\X,\ \forall\, r < h_1,
    \qquad \nu_\X({\cal B}(x, r)) > c_\mu r^q. 
  \]
\end{assumption}

Assumption \ref{ass:mass} is useful to make sure that neighbors in ${\cal D}_n$
are closed with respect to the distance $d$, it is usually derived by assuming
that $\X$ is a subset of $\R^q$; that $\nu_\X$ has a density $p$ against the
Lebesgue measure $\lambda$ with {\em minimal mass} $p_{\min}$ in the sense that
for any $x\in\supp\nu_\X$, $p(x) > p_{\min}$; and that $\supp\nu_\X$ has regular
boundary in the sense that
$\lambda({\cal B}(x, r) \cap \supp\nu_\X) \geq c\lambda({\cal B}(x, r)$
for any $x\in\supp\nu_\X$ and $r < h$ \citep[{\em e.g.,}][]{Audibert2007}.

We now switch to a classical assumption in partial labelling, allowing for
population disambiguation.

\begin{assumption}[Non ambiguity, \citet{Cour2011}]
  \label{ass:non-ambiguity}
  Assume the existence of $\eta \in [0, 1)$, such that for any
  $x\in\supp\nu_\X$, there exists $y_x \in \Y$, such that
  $\Pbb_{\nu}\paren{y_x \in S \midvert X=x} = 1$, and
  \[
    \forall\, z\neq y_x,\quad
    \Pbb_{\nu}\paren{z \in S \midvert X=x} \leq \eta.
  \]
\end{assumption}

Assumption \ref{ass:non-ambiguity} states that when given the full distribution
$\nu$, there is one, and only one, label that is coherent with every observable
sets for a given input. It is a classical assumption
in literature about the learnability of the partial labelling problem
\citep[{\em e.g.,}][]{Liu2014}. When $\ell$ is proper, this implies that
$\mu^*\vert_x = \delta_{y_x}$, and $f^*(x) = y_x$. 

Finally, we assume that $g^*$ is regular. 
As we are considering local averaging method, 
we will use Lipschitz-continuity, 
which is classical in such a setting.\footnote{Its
  generalization through H\"older-continuity would work too.}

\begin{assumption}[Regularity of $g^*$]\label{ass:lipschitz}
  Assume that there exists $c_g > 0$, such that for any $x, x'\in\X$, we have
  \[
    \norm{g^*(x) - g^*(x')}_{\cal H} \leq c_g d(x, x').
  \]
\end{assumption}

It should be noted that regularity of $g^*$, Assumption \ref{ass:lipschitz},
together with determinism of $\mu^*\vert_x$ inherited from
Assumption~\ref{ass:non-ambiguity} implies that classes  
$\X_y = \brace{x\midvert f^*(x) = y}$ are separated in $\X$, in the sense that
there exists $h_2 > 0$, such that for any $y, y' \in \Y$ and 
$(x, x') \in \X_y \times \X_{y'}$, $d(x, x') > h_2$, which is a classical
assumption to derive consistency of semi-supervised learning algorithm 
\citep[{\em e.g.,}][]{Rigollet2007}.
Those implications results from the fact that separation in $\Y$ (hard Tsybakov condition) plus Lipschitzness of $g^*$ implies separation of classes in $\X$, as we details in Appendix~\ref{proof:ass}.

\subsection{Exponential convergence rates}
We are now ready to state our convergence result.
We introduce $h = \min(h_1, h_2)$
and $p = c_\mu h^q$, so that for any $x \in \supp\nu_\X$, $\nu_\X({\cal B}(x, h)) > p$.

\begin{theorem}[Exponential convergence rates]
  \label{thm:convergence}
  When the weights $\alpha$ are given by nearest neighbors,
  under Assumptions \ref{ass:mass}, \ref{ass:non-ambiguity} and
  \ref{ass:lipschitz}, the excess of risk in  Eq.~\eqref{eq:excess} is bounded by
  \begin{align}
    \label{eq:convergence}
    {\cal E}(f_n) &\leq 8c_\psi c_\phi (n+1)\exp\paren{-{\frac{np}{16}}} \hspace*{2cm}
    \nonumber\\&\hspace*{1cm} \qquad + 8c_\psi c_\phi m \exp\paren{-k\abs{\log(\eta)}},
  \end{align}
  as soon as $k < np/ 4$, with $m = \card{\Y}$.
  By taking $k_n = k_0 n$, for $k_0 < p/4$, this implies exponential convergence
  rates ${\cal E}(f_n) = O(n\exp(-n))$.
\end{theorem}
\begin{proof}[Sketch for Theorem \ref{thm:convergence}]
  In essence, based on Lemma \ref{lem:cal}, Theorem \ref{thm:convergence} can be
  understood as two folds. 
  \begin{itemize}
    \item A fully supervised error between $g_n^*$ and $g^*$. This error can
      be controlled in $\exp(-np)$ as the non-ambiguity assumption implies a hard
      Tsybakov margin condition, a setting in which {\em the fully supervised
        estimate $g_n^*$ is known to converge to the population solution $g^*$
        with such rates} \citep{Cabannes2021}. 
    \item A weakly disambiguation error, that is exponential too, since, based
      on Assumption \ref{ass:non-ambiguity}, disambiguating between $z \in \Y$
      and $y_x$ from $k$ sets $S$ sampled accordingly to $\nu\vert_x$ can be
      done in $\eta^k$, and disambiguating between all $z\neq y_x$ and $y_x$
      in $m\eta^k = m \exp(-k\abs{\log(\eta)})$. 
  \end{itemize}
  Appendix~\ref{proof:convergence} provides details.
\end{proof}
Theorem \ref{thm:convergence} states that under a non-ambiguity assumption and a
regularity assumption implying no-density separation, one can expect exponential
convergence rates of $f_n$ learned with weakly supervised data to $f^*$ the
solution of the fully supervised learning problem,
measured with excess of fully supervised risk. 
Because of the exponential convergence, we could derive
polynomial convergence rates for a broader class of problems that are
approximated by problems satisfying assumptions of Theorem \ref{thm:convergence}.
{\em The derived rates in $n\exp(-n)$ should be compared with rates in $n^{-1/2}$ and $n^{-1/4}$},
respectively derived, {\em under the same assumptions}, by \citet{Cour2011,Cabannes2020}.

\subsection{Discussion on Assumptions}
While we have retaken classical assumptions from literature, 
those assumptions are quite strong, which allows us,
by understanding their strength, to derive exponential convergence rates.
Assumptions \ref{ass:mass} and \ref{ass:lipschitz} are classical in the nearest neighbor literature with full supervision. 
If we were using (reproducing) kernel methods to define the weighting scheme $\alpha$, those assumptions would be mainly replaced with 
``$g^*$ belonging to the RKHS''. 
Assumption \ref{ass:non-ambiguity} is the strongest assumption in our view, that we will now discuss.

\paragraph{How to check it in practice ?} 
First, for Assumption \ref{ass:non-ambiguity} to hold, the labels have to be a deterministic function of the inputs. 
In other words, a zero error is achievable. 
Finally, Assumption \ref{ass:non-ambiguity} is related to  dataset collection. 
If dealing with images, weak supervision could take the form of some information on shape, color, or texture, etc.,
Assumption \ref{ass:non-ambiguity} supposes that the weak information potentially given on a specific image~$x$, allows to retrieve the unique label $y$ of the image
({\em e.g.}, a “pig” could be recognized from its shape and its color). 
This is a reasonable assumption, if, for a given $x$, we ask at random a data labeller to provide us information on shape, color, or texture, etc. 
However, it will not be the case, if for some reasons ({\em e.g.} the dataset is built from several weakly annotated datasets), 
in some regions of the input space, we only get shape information, and in other regions, we only get color information. 
In particular, it is not verified for semi-supervised learning when the support of the unlabelled data distribution is not the same as the support of the labelled input data distribution.

\paragraph{How to relax it and what results to expect?} 
Previous works used Assumption \ref{ass:non-ambiguity} to derive a calibration inequality between the infimum loss to the original loss \citep[{\em e.g.}, see Proposition 2 by][]{Cabannes2020}. 
In contrast, we relate the surrogate and original problem through a 
refined calibration inequality \eqref{eq:cal}.
This technical progress allows us to derive exponential convergence rates similarly to the work of \citet{Cabannes2021}.
Importantly, in comparison with previous work, our calibration inequality Lemma \ref{lem:cal} can easily be extended without the determinism assumption provided by Assumption \ref{ass:non-ambiguity}.
Essentially, in our work, Assumption \ref{ass:non-ambiguity} is used to simplify the study of $(\hat{y}_i)_{i\leq n}$ given by the disambiguation algorithm \eqref{eq:disambiguation}, 
and therefore the study of the disambiguation error in Eq. \eqref{eq:cal}.
The study of $(\hat{y}_i)_{i\leq n}$ without Assumption \ref{ass:non-ambiguity} would requires other tools than the one presented in this paper. 
It could be studied in the realm of graphical model and message passing algorithm, or with Wasserstein distance and topological considerations on measures.
With much milder forms of Assumption \ref{ass:non-ambiguity}, we expect the rates to degrade smoothly with respect to a parameter defining the hardness of the problem, similarly to the works of \citet{Audibert2007,Cabannes2021}. 

\section{Optimizaton Considerations}
\label{sec:optim}

In this section, we focus on implementations to solve Eq.~\eqref{eq:disambiguation}.
We explain why disambiguation objectives, such as Eq.~\eqref{eq:principle} are
intrinsically non-convex and express a heuristic strategy to solve
Eq.~\eqref{eq:disambiguation} besides non-convexity in classical well-behaved
instances of partial labelling.
Note that we do not study implementations to solve Eq.~\eqref{eq:estimate} as this
study has already been done by \citet{Nowak2019}.
We end this section by considering a practical example to make derivations more concrete.

\subsection{Non-convexity of disambiguation objectives}

For readability, suppose that $\X$ is a singleton, justifying to remove the 
dependency on the input in the following.
Consider $\nu \in \prob{2^\Y}$ a distribution modelling weak supervision.
While the domain $\brace{\mu\in\prob{\Y} \midvert \mu \vdash \nu}$ is convex, a
disambiguation objective ${\cal E}:\prob{\Y} \to \R$ defining 
$\mu^* \in \argmin_{\mu\vdash\nu} {\cal E}(\mu)$, similarly to
Eq.~\eqref{eq:principle}, that is minimized for deterministic distributions,
which correspond to $\mu$ a Dirac, {\em i.e.}, minimized on vertices of its
definition domain $\prob{\Y}$, can not be convex.
In other terms, any disambiguation objective that pushes toward distributions
where targets are deterministic function of the input, as mentioned in Remark
\ref{rmk:determinism}, can not be convex.

Indeed, smooth disambiguation objectives such as entropy and our piecewise linear
loss-based principle~\eqref{eq:principle}, reading pointwise
\(
  {\cal E}(\mu) = \inf_{z\in\Y} \E_{Y\sim\mu}[\ell(z, Y)],
\)
are concave. Similarly, its quadratic variant
\(
  {\cal E}'(\mu) = \E_{Y, Y'\sim\mu}[\ell(Y, Y')],
\)
is concave as soon as $(\ell(y, y'))_{y, y'\in\Y}$ is semi-definite negative.
We illustrate those considerations on a concrete example with graphical
illustration in Appendix \ref{app:example}.
We should see how this translates on generic implementations to solve
the empirical objective~\eqref{eq:disambiguation}.

\subsection{Generic implementation for Eq.~\eqref{eq:disambiguation}}

Depending on $\ell$ and on the type of observed set $(s_i)$,
Eq.~\eqref{eq:disambiguation} might be easy to solve.
In the following, however, we will introduce optimization considerations to solve
it in a generic structured prediction fashion. To do so, we recall the
decomposition of $\ell$~\eqref{eq:loss} and rewrite
Eq.~\eqref{eq:disambiguation} as 
\begin{equation*}
  (\hat{y}_i)_{i\leq n} \in \argmin_{(y_i) \in C_n} \inf_{(z_i) \in \Y^n}  \sum_{i,j=1}^n \alpha_j(x_i) \psi(z_i)^\top \phi(y_j).
\end{equation*}
Since, given $(y_j)$, the objective is linear in $\psi(z_j)$, the constraint
$\psi(z_j) \in \psi(\Y)$ can be relaxed with $\zeta_i \in \hull\psi(\Y)$.\footnote{The
  minimization pushes towards extreme points of the definition domain.}
Similarly, with respect to $\phi(y_j)$, 
this objective is the infimum of linear functions, therefore is
concave, and the constraint $\phi(y_j) \in \phi(s_j)$, could be relaxed with
$\xi_i \in \hull\phi(s_j)$. Hence, with ${\cal H}_0 = \hull\psi(\Y)$ and
$\Gamma_n = \prod_{j\leq n} \hull\phi(s_j)$, the optimization is cast as
\begin{equation}
  \label{eq:algo}
  (\hat{\xi}_i)_{i\leq n} \in \argmin_{(\xi_i) \in \Gamma_n} 
  \inf_{(\zeta_i) \in{\cal H}_0^n} \sum_{i,j=1}^n \alpha_j(x_i) \zeta_i^\top \xi_j.
\end{equation}
Because of concavity, $(\hat{\xi}_i)$ will be an extreme point of
$\Gamma_n$, that could be decoded into $\hat{y}_i = \phi^{-1}(\hat{\xi}_i)$.
However, it should be noted that if only interested in $f_n$ and not in the
disambiguation $(\hat{y}_i)$, this decoding can be avoided, since
Eq.~\eqref{eq:estimate} can be rewritten as
$f_n(x) \in \argmin_{z\in\Y} \psi(z)^\top \sum_{i=1}^n \alpha_i(x) \hat{\xi}_i$.

\subsection{Alternative minimization with good initialization}

To solve Eq.~\eqref{eq:algo}, we suggest to use an alternative minimization
scheme. The output of such an scheme is highly dependent to the variable
initialization.
In the following, we introduce well-behaved problem, where $(\xi_i)_{i\leq n}$
can be initialized smartly, leading to an efficient implementation to solve
Eq.~\eqref{eq:algo}.

\begin{definition}[Well-behaved partial labelling problem]
  \label{def:init}
  A partial labelling problem $(\ell, \nu)$ is said to be well-behaved if for
  any $s \in \supp\nu_{2^\Y}$, there exists a signed measure $\mu_s$ on $\Y$
  such that the function from $\Y$ to $\R$ defined as 
  $z\to \int_{\Y} \ell(z, y) \diff\mu_s(y)$ is minimized for, and only for, 
  $z \in s$.
\end{definition}

We provide a real-world example of a well-behaved problem in Section
\ref{sec:ranking} as well as a synthetic example with graphical illustration in
Appendix \ref{app:example}.
On those problems, we suggest to solve Eq.~\eqref{eq:algo} by considering the initialization
$\xi^{(0)}_i = \E_{Y\sim\mu_{s_i}}[\phi(Y)]$, and performing alternative
minimization of Eq.~\eqref{eq:algo}, until attaining $\xi^{(\infty)}$ as the
limit of the alternative minimization scheme (which exists since each step
decreases the value of the objective in Eq.~\eqref{eq:algo} and there is a
finite number of candidates for $(\xi_i)$). It corresponds to a disambiguation
guess $\tilde{y}_i = \phi^{-1}(\xi_i^{(\infty)})$. Then we suggest to learn
$\hat{f}_n$ from $(x_i, \tilde{y}_i)$ based on Eq.~\eqref{eq:estimate}, and
existing algorithmic tools for this problem \citep{Nowak2019}.
To assert the well-groundedness of this heuristic, we refer to the following
proposition, proven in Appendix \ref{proof:init}.

\begin{proposition}
  \label{prop:init}
  Under the non-ambiguity hypothesis, Assumption \ref{ass:non-ambiguity},
  the solution of Eq.~\eqref{eq:solution} is characterized by
  \(
  f^* \in \argmin_{f:\X\to\Y} \E_{(X, S) \sim \nu}\bracket{\E_{Y\sim\mu_S}[\ell(f(X), Y)]}.
  \)
  Moreover, if the surrogate function $g_n^{\circ}:\X\to{\cal H}$ defined as
  $g_n^{\circ}(x) = \sum_{i=1}^n\alpha_i(x)\xi_{s_i}$,
  with $\xi_s = \E_{Y\sim\mu_s}[\phi(Y)]$,
  converges towards $g^\circ(x) = \E_{S\sim\nu\vert_x}[\xi_S]$ in $L^1$,
  $f_n^{\circ}$ defined through the decoding Eq.~\eqref{eq:decoding} converges
  in risk towards $f^*$.
\end{proposition}

Given that our algorithm scheme is initialized for
$\xi_i^{(0)} = \xi_{s_i}$ and $\zeta_i^{(0)} = f_n^{\circ}(x_i)$ and stopped
once having attained $\xi_i^{(\infty)}$ and $\zeta_i^{(\infty)} = \hat{f}_n(x_i)$,
$\hat{f}_n$ is arguably better than $f_n^{\circ}$, which given
consistency result exposed in Proposition \ref{prop:init}, is already good enough.

\begin{figure*}[t]
  \centering
  \vskip -0.1in
  \includegraphics{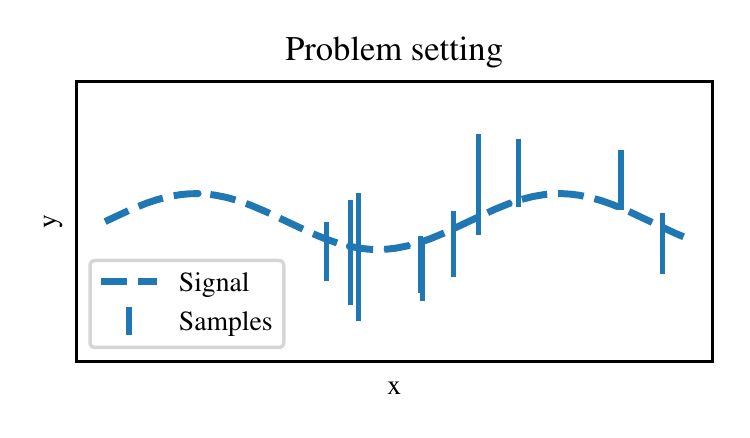}
  \includegraphics{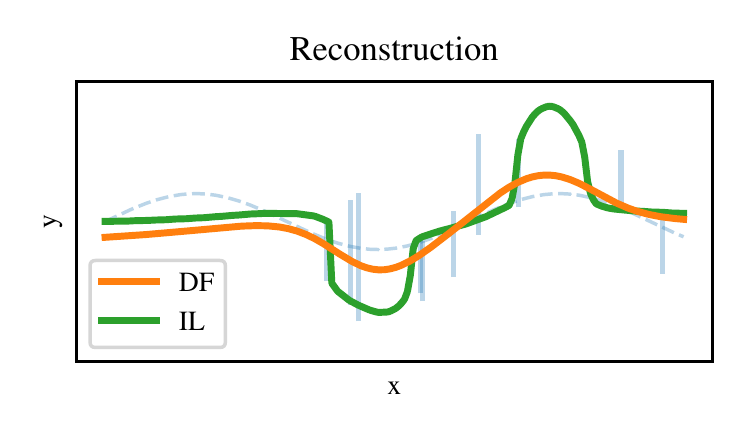}
  \vskip -0.2in
  \caption{
    Interval regression. See Appendix \ref{app:experiments} for the exact
    reproducible experimental setup (Left) Setup. The goal is to learn $f^*:\X\to\R$
    represented by the dashed line, given samples $(x_i, s_i)$, where $(s_i)$
    are intervals represented by the blue segments. 
    (Right) We compare the Infimum Loss (IL) baseline~\eqref{eq:baseline} shown
    in green, with our Disambiguation Framework (DF), 
    Eqs.~\eqref{eq:disambiguation} and \eqref{eq:estimate}, shown in orange;
    with weights $\alpha$ given by kernel ridge regression. (DF) retrieves
    $\hat{y_i}$ before learning a smooth $f_n$ based on $(x_i, \hat{y}_i)$,
    while (IL) implicitly retrieves $\hat{y}_i(x)$ differently for each input,
    leading to irregularity of the consequent estimator of $f^*$.
  }
  \label{fig:ir}
  \vskip -0.1in
\end{figure*}

\begin{remark}[IQP implementation for Eq.~\eqref{eq:disambiguation}]
  Other heuristics to solve Eq.~\eqref{eq:disambiguation} are conceivable.
  For example, considering $z_i = y_i$ in this equation, we remark that
  the resulting problem is isomorphic to an integer quadratic program (IQP).  
  Similarly to integer linear programming, this problem can be approached with
  relaxation of the ``integer constraint'' to get a real-valued solution,  
  before ``thresholding'' it to recover an integer solution. 
  This heuristic can be seen as a generalization of the Diffrac algorithm
  \citep{Bach2007,Joulin2010}.
  we present it in details in Appendix~\ref{app:diffrac}.
\end{remark}

\begin{remark}[Link with EM, \citep{Dempster1977}]
  Arguably, our alternative minimization scheme, optimizing respectively the
  targets $\xi_i = \phi(y_i)$ and the function estimates $\zeta_i = \psi(f_n(x_i))$
  can be seen as the non-parametric version of the Expectation-Maximization
  algorithm, popular for parametric model \citep{Dempster1977}. 
\end{remark}

\subsection{Application: Ranking with partial ordering}
\label{sec:ranking}

Ranking is a problem consisting, for an input $x$ in an input space $\X$, to
learn a total ordering $y$, belonging to $\Y = \Sfrak_m$, modelling preference
over $m$ items. It is usually approach with the Kendall loss
$\ell(y, z) = - \phi(y)^\top \phi(z)$, with
$\phi(y) = (\sign\paren{y(i) - y(j)})_{i,j\leq m} \in \brace{-1,1}^{m^2}$ \citep{Kendall1938}.
Full supervision corresponds, for a given $x$, to be given a total ordering of
the $m$ items. This is usually not an option, but one could expect to be given
partial ordering that $y$ should follow \citep{Cao2007,Hullermeier2008,Korba2018}.
Formally, this equates to the observation of some, but not all, coordinates
$\phi(y)_i$ of the vector $\phi(y)$ for some $i\in I \subset \bbracket{1, m}^2$.

In this setting, $s\subset \Y$ is a set of total orderings that match the given
partial ordering. It can be represented by a vector $\xi_s \in {\cal H}$, that
satisfies the partial ordering observation, $(\xi_s)_I = \phi(y)_I$, and that is
agnostic on unobserved coordinates, $(\xi_s)_{^c I} = 0$.
This vector satisfies that $z\to\psi(z)^\top\xi_s$ is minimized for, and only
for, $z\in s$.  Hence, it constitutes a good initialization for the alternative
minimization scheme detailed above. We provide details in Appendix
\ref{proof:ranking}, where we also show that $\xi_s$ can be formally translated
in a $\mu_s$ to match the Definition \ref{def:init}, proving that ranking with
partial labelling is a well-behaved problem. 

Many real world problems can be formalized as a ranking problem with partial
ordering observations. For example, $x$ could be a social network user, and the
$m$ items could be posts of her connection that the network would like to order
on her feed accordingly to her preferences. One might be told that the user $x$
prefer posts from her close rather than from her distant connections, which
translates formally as the constraint that for any $i$ corresponding to a post
of a close connection and $j$ corresponding to a post of a distant connection,
we have $\phi(y)_{ij} = 1$.
Nonetheless, designing non-parametric structured prediction models that scale
well when the intrinsic dimension $m$ of the space $\Y$ is very large (such as
the number of post on a social network) remains an open problem, that this paper
does not tackle. 

\section{Related work}
\label{sec:litterature}

Weakly supervised learning has been approached through parametric and
non-parametric methods. Parametric models are usually optimized through maximum
likelihood \citep{Heitjan1991,Jin2002}.
\citet{Hullermeier2014} show that this approach, as formalized by
\citet{Denoeux2013}, equates to disambiguating sets by averaging candidates,
which was shown inconsistent by \citet{Cabannes2020} when data are {\em not
  missing at random}.
Among non-parametric models, \citet{Xu2004,Bach2007} developed an algorithm for
clustering, that has been cast for weakly supervised learning problem
\citep{Joulin2010,Alayrac2016}, leading to a disambiguation algorithm similar than ours, yet
without consistency results.
More recently, half-way between theory and practice, \citet{Gong2018} derived an
algorithm geared towards classification, based on a disambiguation objective,
incorporating several heuristics, such as class separation, and Laplacian
diffusion.  
Those heuristics could be incorporated formally in our model.

The infimum loss principle has been considered by several authors, among them
  \citet{Cour2011,Luo2010,Hullermeier2014}. It was recently analyzed through
the prism of structured prediction by \citet{Cabannes2020}, leading to a
consistent non-parametric algorithm that will constitute the baseline of our
experimental comparison.
This principle is interesting as it does not assume knowledge on the corruption
process $(S\vert Y)$ contrarily to the work of \citet{CidSueiro2014} or
\citet{vanRooyen2018}. 

The non-ambiguity assumption has been introduced by \citet{Cour2011} and is a
classical assumption of learning with partial labelling
\citep{Liu2014}. Assumptions of Lipschitzness and minimal mass are
classical assumptions to prove convergence of local averaging method
\citep{Audibert2007,Biau2015}.
Those assumptions imply class separation in $\X$, which has been leverage in
semi-supervised learning, justifying Laplacian regularization \citep{Rigollet2007,Zhu2003}.

Note that those assumptions might not hold on raw representation of the data,
but with appropriate metrics, which could be learned through unsupervised
\cite{Duda2000} or self-supervised learning \cite{Doersch2017}. 
Indeed, \citet{Wei2021} provide an analysis akin ours based on such an assumption.
As such, the practitioner might consider weights $\alpha$ given by similarity
metrics derived through such techniques, before computing the
disambiguation~\eqref{eq:disambiguation} and learning $f_n$ from the recollected
fully supervised dataset with deep learning. 

\begin{figure*}[t]
  \centering
  \vskip -0.1in
  \includegraphics{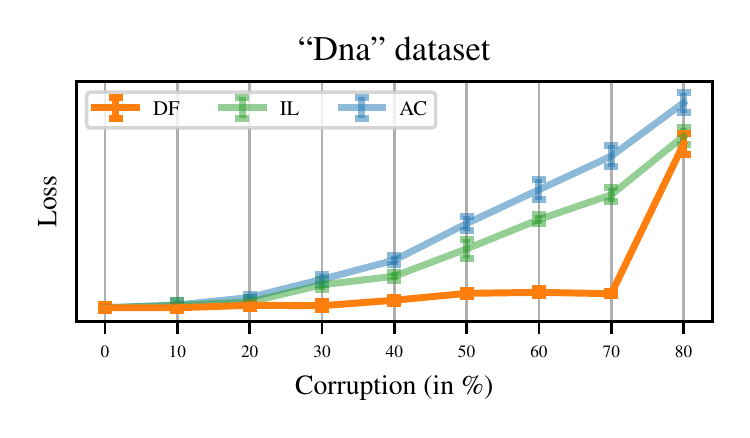}
  \includegraphics{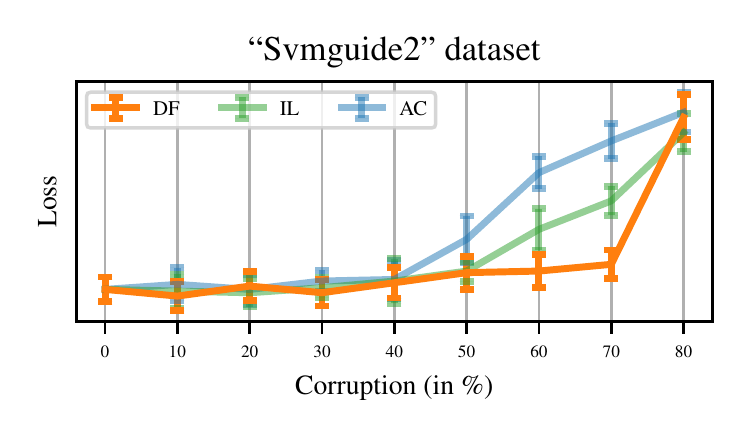}
  \vskip -0.2in
  \caption{Testing errors as function of the supervision corruption on real
    dataset corresponding to classification with partial labels.
    We split fully supervised LIBSVM datasets into training and testing dataset.
    We corrupt training data in order to get partial labels. Corruption is
    managed through a parameter, represented by the $x$-axis, that relates to
    the ambiguity degree $\eta$ of Assumption \ref{ass:non-ambiguity}. For each
    methods (our algorithm (DF), the baseline (IL), and the baseline of the
    baseline (AC, consisting of averaging candidates $y_i$ in sets $S_i$)), we consider
    weights $\alpha$ given by kernel ridge regression with 
    Gaussian kernel, for which we optimized hyperparameters with cross-validation
    on the training set. We then learn an estimate $f_n$ that we evaluate on the
    testing set, represented by the $y$-axis, on which we have full supervision.
    The figure show the superiority of our method, that achieves error similar to
    baseline when full supervision ($x=0$) or no supervision ($x=100\%$) is
    given, but performs better when only in presence of partial supervision.
    See Appendix \ref{app:experiments} for reproducibility specifications, where
    we also provide Figure \ref{fig:rk} showcasing similar empirical results in
    the case of ranking with partial ordering. 
  }
  \label{fig:cl}
  \vskip -0.1in
\end{figure*}

\section{Experiments}
\label{sec:example}

In this section, we review a baseline, and experiments that showcase the
usefulness of our algorithm
Eqs.~\eqref{eq:disambiguation}~and~\eqref{eq:estimate}.

\paragraph{Baseline.}
We consider as a baseline the work of \citet{Cabannes2020}, which is a
consistent structured prediction approach to partial labelling through the
infimum loss. It is arguably the state-of-the-art of partial labelling approached
through structured prediction.
It follow the same loss-based variance disambiguation principle,
yet in an implicit fashion, leading to the inference algorithm, $f_n:\X\to\Y$,
\begin{equation}
  \label{eq:baseline}
  f_n(x) \in \argmin_{z\in\Y} \inf_{(y_i) \in C_n} \sum_{i=1}^n \alpha_i(x) \ell(z, y_i).
\end{equation}
Note that with our proof technique, which overcome the sub-optimality of calibration inequality \citep{Audibert2007,Cabannes2021}, exponential convergence rates similar to Theorem
\ref{thm:convergence} could be derived for the baseline. 
Yet, as we will see, our algorithm outperforms this state-of-the-art baseline.
This could be explained by the fact that our algorithm introduce an intrinsically smaller surrogate space (in essence, \citet{Cabannes2020} introduced surrogate functions from inputs in $\X$ to powersets represented in $\R^{2^\Y}$, while we look at functions from input in $\X$ to output represented in $\R^\Y$).

\paragraph{Disambiguation coherence - Interval regression.}
The baseline Eq.~\eqref{eq:baseline} implicitly requires to disambiguate
$(\hat{y}_i(x))$ differently for every $x\in\X$. This is counter intuitive since
$(y_i^*)$ does not depend on $x$. It means that $(\hat{y}_i)$ could be equal to some
$(\hat{y}_i^{(0)})$ on a subset $\X_0$ of $\X$, and to another
$(\hat{y}_i^{(1)})$ on a disjoint subset $\X_1 \subset \X$, leading to
irregularity of $f_n$ between $\X_0$ and $\X_1$. We illustrate this graphically
on Figure \ref{fig:ir}. This figure showcases an interval regression problem,
which corresponds to the regression setup
($\Y = \R$, $\ell(y, z) = \abs{y - z}^2$) of partial labelling, where one does
not observed $y\in\R$ but an interval $s\subset \R$ containing $y$.
Among others, this problem appears in physics \citep{Sheppard1897} and economy \citep{Tobin1958}.

\paragraph{Computation attractiveness - Ranking.}
Computationally, the baseline requires to solve a disambiguation
problem, recovering $(\hat{y_i}(x)) \in C_n$ for every $x\in\X$ for which we
want to infer $f_n(x)$. 
This is much more costly, than doing the disambiguation of $(\hat{y_i}) \in C_n$
once, and solving the supervised learning inference problem
Eq.~\eqref{eq:estimate}, for every $x\in\X$ for which we want to infer $f_n(x)$.
To illustrate the computation attractiveness of our algorithm, consider the
case of ranking, defined in Section \ref{sec:ranking}.
Fully supervised inference scheme~\eqref{eq:estimate} corresponds to solving a
NP-hard problem, equivalent to the minimum feedback arcset problem
\citep{Duchi2010}.
While disambiguation approaches with alternative minimization implied by
Eq.~\eqref{eq:disambiguation} and Eq.~\eqref{eq:baseline} require to solve this
NP-hard problem for each minimization step.
In other terms, the baseline ask to solve multiple NP-hard problem every time
one wants to infer $f_n$ given by Eq.~\eqref{eq:baseline} on an input $x\in\X$.
Meanwhile, our disambiguation approach asks to solve multiple NP-hard problem
upfront to solve Eq.~\eqref{eq:disambiguation}, yet only require to solve one
NP-hard problem to infer $f_n$ given by Eq.~\eqref{eq:estimate} on an input $x\in\X$.

\paragraph{Better empirical results - Classification.}
Finally, we compare our algorithm, our baseline~\eqref{eq:baseline} and the
baseline considered by \citet{Cabannes2020} on real datasets from the LIBSVM dataset
\citep{Chang2011}. Those datasets $(x_i, y_i)$ correspond to fully supervised classification
problem. In this setup, $\Y = \bbracket{1,m}$ for $m$ a number of classes, and
$\ell(y, z) = \ind{y\neq z}$. We ``corrupt'' labels in order to create a synthetic
weak supervision datasets $(x_i, s_i)$. We consider skewed
corruption, in the sense that $(s_i)$ is generated by a probability such that
$\sum_{z\in\Y} \Pbb_{S_i}(z\in S_i \vert y_i)$ depends on the value of
$y_i$. This corruption is parametrized by a parameter that related with the
ambiguity parameter $\eta$ of Assumption \ref{ass:non-ambiguity}.
Results on Figure \ref{fig:cl} show that, in addition to having
a lower computation cost, our algorithm performs better in practice than the
state-of-the-art baseline.\footnote{All the code is available online - \url{https://github.com/VivienCabannes/partial_labelling}.}

\begin{figure*}[t]
  \centering
  \vskip -0.1in
  \includegraphics{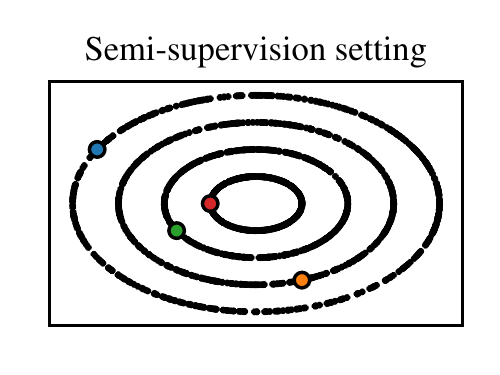}
  \includegraphics{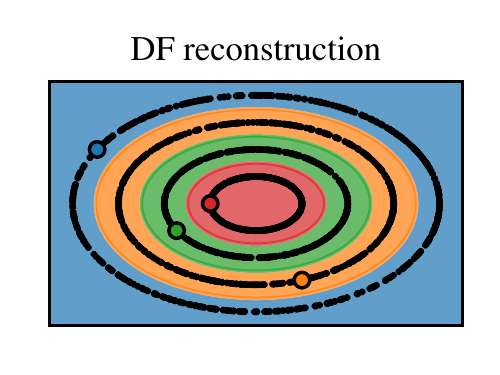}
  \includegraphics{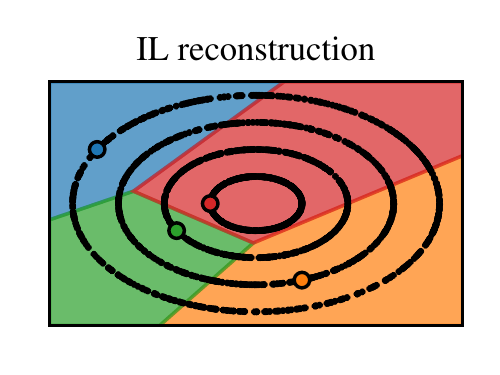}
  \vskip -0.2in
  \caption{Semi-supervised learning, ``concentric circle'' instance with four
    classes (red, green, blue, yellow). Reproducibility details provided in
    Appendix \ref{app:experiments}. (Left) We represent points $x_i \in \X
    \subset \R^2$, there is many unlabelled points (represented by black dots
    and corresponding to $S_i = \Y$), and one labelled point for each class
    (represented in color, corresponding to $S_i = \brace{y_i}$). (Middle)
    Reconstruction $f_n:\X \to \Y$ given by our algorithm
    Eqs.~\eqref{eq:disambiguation} and \eqref{eq:estimate}. Our algorithm
    succeeds to comprehend the concentric circle structure of the input
    distribution and clusters classes accordingly. (Right) Reconstruction
    $f_n:\X\to\Y$ given by the baseline Eq.~\eqref{eq:baseline}. The baseline
    performs as if only the four supervised data points where given.} 
  \label{fig:ss}
  \vskip -0.1in
\end{figure*}

\paragraph{Beyond Eq.~\eqref{eq:principle} - Semi-supervised learning.} 
The main limitation of Eq.~\eqref{eq:principle} is that it is a pointwise principle
that decorrelates inputs, in the sense that the optimization of $\mu^*\vert_x$,
for $x\in\X$, only depends on $\nu\vert_x$ and not on what is happening on
$\X\backslash \brace{x}$. As such, this principle failed to tackle
semi-supervised learning, where $\nu\vert_x$ is equal to $\mu\vert_x$ (in the
sense that $\pi\vert_{x, y} =\delta_{\brace{y}}$) for $x\in\X_l$ and is equal to
$\delta_\Y$ for $x \in  \X_u:=\X\backslash\X_l$. In such a setting, for
$x\in\X_u$, $\mu^*\vert_x$ can be set to any $\delta_y$ for $y\in\Y$. 
Interestingly, in practice, while the baseline suffer the same limitation,
for our algorithm, {\em weighting schemes have a regularization effect}, that
contrasts with those considerations. We illustrate it on Figure \ref{fig:ss}. 

\paragraph{Real real-world applications.}
There is a real lack of clean datasets to experiment with partial labelling. 
Most theoretical papers consist in synthetic corruption of fully supervised dataset \citep[\emph{e.g.,}][]{Korba2018} as we did. 
Empirical papers are built on highly complex datasets that require skilled pre-processing and tricks beside theoretically-grounded principle ({\em e.g.}, action recognition on Youtube videos). 
However, note that the state-of-the-art work of \citet{Miech2019} is built on heuristics from the Diffrac algorithm, which we generalized \citep[see][for details]{Alayrac2016}. 
We hope that, by providing theoretical understanding of the problem, our paper could help to design powerful heuristics in practice, even though this is out of scope of the present paper.

\section{Conclusion}
\label{sec:opening}

In this work, we have introduced a structured prediction algorithm
Eqs.~\eqref{eq:disambiguation} and \eqref{eq:estimate}, to tackle partial
labelling. We have derived exponential convergence rates for the nearest
neighbors instance of this algorithm under classical learnability assumptions.
We provided optimization considerations to implement this algorithm in practice,
and have successfully compared it with the state-of-the-art.
Several open problems offer prospective follow-up of this works.
\begin{itemize}
    \item 
\emph{Semi-supervised learning and beyond.}
While we only proved convergence in situation where $\mu^*$ of
Eq.~\eqref{eq:principle} is uniquely defined, therefore excluding
semi-supervised learning, Figure \ref{fig:ss} suggests that our algorithm
\eqref{eq:disambiguation} could be analyzed in a broader setting than the one
considered in this paper. 
Among others, the non-ambiguity assumption could
be replaced by a cluster assumption \citep{Rigollet2007} together with a
non-ambiguity assumption cluster-wise in Theorem~\ref{thm:convergence}. 
\item
\emph{Hard-coded weak supervision.}
Variational principles Eqs.~\eqref{eq:principle} and
\eqref{eq:solution} could be extended beyond partial labelling to any type of
hard-coded weak supervision, which is when weak supervision can be cast as a
set of hard constraint that $\mu$ should satisfy, formally written as a set of
fully supervised distributions compatible with weak information. Hard-coded weak
supervision includes label proportion \citep{Quadrianto2009,Dulac2019}, but
excludes supervision of the type
``80\% of the experts say this nose is broken, and 20\% say it is not''.
Providing a unifying framework for those problems would make an
important step in the theoretical foundation of weakly supervised learning.
\item
\emph{Missing input data.}
While weak supervision assumes that only $y$ is partially known, in many
applications of machine learning, $x$ is also only partially known, especially
when the feature vector $x$ is built from various source of information, leading
to missing data. While we only considered a principle to fill missing output
information, similar principles could be formalized to fill missing input information. 
This would be particularly valuable when data are not missing at random \citep{Rubin1976,Muzellec2020}.

\end{itemize}


\section*{Acknowledgements}
We would like to thanks anonymous referees for helpful comments.
This work was funded in part by the French government
under management of Agence Nationale de la Recherche as part of the
``Investissements d'avenir'' program, reference ANR-19-P3IA-0001 (PRAIRIE 3IA
Institute). We also acknowledge support of the European Research Council
(grants SEQUOIA 724063, REAL 94790). 

\bibliography{main}
\bibliographystyle{style/icml2021}

\onecolumn
\appendix

\section{Proofs}

\paragraph{Mathematical assumptions.}
To make formal what should be seen as implicit assumptions heretofore,
we consider $\X$ and $\Y$ Polish spaces, $\Y$ compact,
$\ell:\Y\times\Y\to\R$ continuous, ${\cal H}$ a separable Hilbert space, $\phi$
measurable, and $\psi$ continuous.
We also assume that for $\nu_x$-almost every $x\in\X$, and any $\mu\vdash\nu$,
that the pushforward measure $\phi_*\mu\vert_x$ has a second moment.
This is the sufficient setup in order to be able to define formally objects and
solutions considered all along the paper.

\paragraph{Notations.}
Beside standard notations, we use $\card{\Y}$ to design the cardinality of $\Y$,
and $2^\Y$ to design the set of subsets of $\Y$.
Regarding measures, we use $\mu_\X$ and $\mu\vert_x$ respectively the marginal
over $\X$ and the conditional accordingly to $x$ of $\mu\in\prob{\X\times\Y}$.
We denote by $\mu^{\otimes n}$ the distribution of the random variable $(Z_1, \cdots, Z_n)$,
where the $Z_i$ are sampled independently according to $\mu$.
For $A$ a Polish space, we consider $\prob{A}$ the set of Borel
probability measures on this space.
For $\phi:\Y\to{\cal H}$ and $S\subset \Y$, we denote by $\phi(S)$ the set $\brace{\phi(y)\midvert y\in S}$.
For a family of sets $(S_i)$, we denote by $\prod S_i$ the Cartesian product $S_1\times S_2\times\cdots$,
also defined as the set of points $(y_i)$ such that $y_i \in S_i$ for all index $i$,
and by $\Y^n$ the Cartesian product $\prod_{i\leq n}\Y$. 
Finally, for $E$ a subset of a vector space $E'$, $\hull E$ denotes the convex
hull of $E$ and $\Span(E)$ its span.

\paragraph{Abuse of notations.}
For readability sake, we have abused notations.
For a signed measure $\mu$, we denote by $\E_{\mu}[X]$ the integral $\int
x\diff\mu(x)$, extending this notation usually reserved to probability measure.
More importantly, when considering $2^\Y$, we should actually restrict ourselves
to the subspace ${\cal S} \subset 2^\Y$ of closed subsets of $\Y$, as ${\cal S}$
is a Polish space (metrizable by the Hausdorff distance) while $2^\Y$ is not always.
However, when $\Y$ is finite, those two spaces are equals, $2^\Y = {\cal S}$.

\subsection{Proof of Lemma \ref{lem:cal}}
\label{proof:cal}
 
  From Lemma 3 in \citet{Cabannes2021}, we pulled the calibration inequality
  \[
    {\cal R}(f_n) - {\cal R}(f^*) \leq 2c_\psi
    \E\bracket{\ind{\norm{g_n(X) - g^*(X)} > d(g^*(X), F)}\norm{g_n(X) - g^*(X)}}.
  \]
  Where $F$ is defined as the set of points $\xi \in \hull\phi(\Y)$ leading to
  two decodings
  \[
    F = \brace{\xi \in \hull\phi(\Y)\midvert
      \card{\argmin_{z\in\Y}\scap{\psi(z)}{\xi}} > 1},
  \]
  and $d$ is defined as the extension of the norm distance to sets, for $\xi \in
  {\cal H}$
  \[
    d(\xi, F) = \inf_{\xi'\in F} \norm{\xi - \xi'}_{\cal H}.
  \]
  Using that $\norm{g_n(X) - g^*(X)} \leq \norm{g_n(X) - g_n^*(X)} +
  \norm{g_n^*(X) - g^*(X)}$ and that, if $a \leq b + c$,
  \[
    \ind{a > \delta} a \leq \ind{b + c > \delta} b + c \leq \ind{2\sup(b, c) >
      \delta} 2 \sup{b, c} = 2 \sup_{e\in{b, c}} \ind{e > \delta} e
    \leq 2 \ind{b>\delta} b + 2\ind{c > \delta} c.
  \]
  We get the refined inequality
  \[
    {\cal R}(f_n) - {\cal R}(g^*) \leq 4 c_\psi
    \E\bracket{\ind{2\norm{g_n(X) - g_n^*(X)} >  d(g^*(X), F)}\norm{g_n(X) -
        g_n^*(X)}
    + \ind{2\norm{g_n^*(X) - g^*(X)} >  d(g^*(X), F)}\norm{g_n^*(X) - g^*(X)}}.
  \]
  The first term is bounded with
  \[
    \E\bracket{\ind{2\norm{g_n(X) - g_n^*(X)} >  d(g^*(X), F)}\norm{g_n(X) -
        g_n^*(X)}} \leq \norm{g_n - g_n^*}_{L^1}.
  \]
  While for the second term, we proceed with
  \[
    \E\bracket{\ind{2\norm{g_n^*(X) - g^*(X)} >  d(g^*(X), F)}\norm{g_n^*(X)
        -g^*(X)}}
    \leq \norm{g_n^* - g^*}_{L^\infty} \Pbb_X\paren{2\norm{g_n^*(X) - g^*(X)} >
      \inf_{x\in\supp\nu_\X} d(g^*(X), F)}.
  \]
  When weights sum to one, that is $\sum_{i=1}^n\alpha_i(X) = 1$, both
  $g_n^*(X)$ and $g^*(X)$ are averaging of $\phi(y)$ for $y\in\Y$, therefore
  \[
    \norm{g_n^* - g^*}_{L^\infty} \leq 2 c_\phi.
  \]
  Finally, when the labels are a deterministic function of the input, $g^*(X) =
  \phi(f^*(X))$, and $d(g^*(X), F) \leq \sup_{y\in\Y} d(\phi(y), F)$.
  Defining $\delta := \sup_{y\in\Y} d(\phi(y), F) /2$, and adding everything
  together leads to Lemma~\ref{lem:cal}.

\subsection{Implication of Assumptions \ref{ass:non-ambiguity} and
  \ref{ass:lipschitz}}
\label{proof:ass}

  Assume that Assumption \ref{ass:non-ambiguity} holds, consider
  $x\in\supp\nu_\X$, let us show that $f^*(x) = y_x$ and 
  $\mu^*\vert_x = \delta_{y_x}$. First of all, notice that 
  \(\bigcap_{S;S\in\supp\nu\vert_x} = \brace{y_x} \);
  that $\delta_{y_x} \vdash \nu\vert_x$, as it corresponds to
  $\pi\vert_{x, S} = \delta_{y_x} \in \prob{S}$, for all $S$ in the support of
  $\nu\vert_{x}$; and that, because $\ell$ is well-behaved,
  \[
    \inf_{z\in\Y} \ell(z, y_x) = \ell(y_x, y_x) = 0.
  \]
  This infimum is only achieved for $z = y_x$, hence if we prove that
  $\mu^*\vert_x = \delta_{y_x}$, we directly have that $f^*(x) = y_x$.
  Finally, suppose that $\mu\vert_x\vdash\nu\vert_x$ charges $y \neq y_x$. Because
  $y$ does not belong to all sets charged by $\nu\vert_x$, $\mu\vert_x$ should
  charge an other $y'\in\Y$, and therefore
  \[
    \inf_{z\in\Y} \E_{Y\sim\mu\vert_x}[\ell(z, y)] \geq
    \inf_{z\in\Y}\mu\vert_x(y) \ell(z, y) + 
    \mu\vert_x(y') \ell(z, y') > 0.
  \]
  Which shows that $\mu^*\vert_x = \delta_{y_x}$. 
  We deduce that $g^*(x) = y_x$.

  Now suppose that Assumption \ref{ass:lipschitz} holds too, and consider two 
  $x, x' \in\supp\nu_\X$ belonging to two different classes 
  $f(x) = y$ and $f(x') = y'$. We have that $g^*(x) = \phi(y)$ and $g^*(x') =
  \phi(y')$, therefore,
  \[
    d(x, x') \geq c^{-1} \norm{\phi(y) - \phi(y')}_{\cal H}.
  \]
  Define $h_2 = \inf_{y\neq y'} c^{-1} \norm{\phi(y) - \phi(y')}_{\cal H}$. Let us
  now show that $h_2 > 0$. When $\Y$ is finite, this infimum is a minimum,
  therefore, $h_2 = 0$, only if there exists a $y \neq y'$, such that 
  $\phi(y) = \phi(y')$, which would implies that 
  $\ell(\cdot, y) = \ell(\cdot,y')$ and therefore $\ell(y, y') = \ell(y, y)$ which
  is impossible when $\ell$ is proper. 

\subsection{Proof of Theorem \ref{thm:convergence}}
\label{proof:convergence}

  Reusing Lemma \ref{lem:cal}, we have
  \[
    {\cal E}(f_n) \leq 4c_\psi \E_{{\cal D}_n, X}\bracket{\norm{g_n^*(X) - g_n(X)}_{\cal H}} +
    8c_\psi c_\phi \E_{{\cal D}_n, X}\bracket{\ind{\norm{g_n^*(X) - g^*(X)} > \delta}}. 
  \]
  We will first prove that 
  \[
    \E_{{\cal D}_n}\bracket{\ind{\norm{g_n^*(X) - g^*(X)} > \delta}}
    \leq \exp\paren{-\frac{np}{8}}
  \]
  as long as $k < np / 2$.
  The error between $g^*$ and $g_n$ relates to classical supervised learning of
  $g^*$ from samples $(X_i, Y_i) \sim \mu^*$. We invite the
  reader who would like more insights on this fully supervised part of the proof
  to refer to the several monographs written on local averaging methods and, in
  particular, nearest neighbors, such as \citet{Biau2015}. 
  Because of class separation, we
  know that, if $k$ points fall at distance at most $h$ of $x \in \supp\nu_\X$,
  $g_n^*(x) = k^{-1}\sum_{i; X_i\in{\cal N}(x)} \phi(y_i) = \phi(y_x) = g^*(x)$, where
  ${\cal N}(x)$ designs the $k$-nearest neighbors of $x$ in $(X_i)$. Because the
  probability of falling at distance $h$ of $x$ for each $X_i$ is lower bounded
  by~$p$, we have that
  \[
    \Pbb_{{\cal D}_n}(g_n^*(x) \neq g^*(x)) \leq \Pbb(\text{Bernouilli}(n, p) < k). 
  \]
  This can be upper bound by $\exp(- np / 8)$ as soon as $k < np/2$, based on
  Chernoff multiplicative bound \citep[see][for a reference]{Biau2015}, meaning
  \[
    \E_{{\cal D}_n, X} \bracket{\ind{\norm{g_n^*(X) - g^*(X)} \geq \delta}}
    \leq \exp(- np / 8).
  \]
  
  For the disambiguation part in $\norm{g_n - g_n^*}_{L^1}$, we distinguish two
  types of datasets, the ones where for any input $X_i$ its $k$-neighbors at are
  distance at least $h$, ensuring that disambiguation can be done by clusters, and
  datasets that does not verify this property.
  Consider the event
  \[
    \mathbb{D} = \brace{(X_i)_{i\leq n} \midvert \sup_{i} d(X_i, X_{(k)}(X_i)) < h}
  \]
  where $X_{(k)}(x)$ design the $k$-th nearest neighbor of $x$ in
  $(X_i)_{i\leq n}$. We proceed with
  \[
    \E_{{\cal D}_n, X}\bracket{\norm{g_n^*(X) - g_n(X)}_{\cal H}}
    \leq \sup_{X\in\X} \norm{g_n^* - g_n}_{\infty} \Pbb_{{\cal
        D}_n}((X_i) \notin\mathbb{D})
    + \E_{{\cal D}_n, X}\bracket{\norm{g_n^*(X) - g_n(X)}_{\cal H}\midvert (X_i) \in \mathbb{D}},
  \]
  Which is based on $E[Z] = \Pbb(Z\in A)\E[Z\vert A] + \Pbb(Z\notin A)\E[Z\vert ^cA]$.
  For the term corresponding to bad datasets, we can bound the
  disambiguation error with the maximum error.
  Similarly to the derivation for Lemma
  \ref{lem:cal}, because $g_n^*(x)$ and $g_n^*(X)$, are averaging of $\phi(y)$, we have
  that
  \[
    \sup_{x\in\supp\nu_\X} \norm{g_n(x) - g_n^*(x)} \leq 2c_\phi.
  \]
  Indeed, we allow ourselves to pay the worst error on those datasets as their
  probability is really small, which can be proved based on the following derivation.
  \begin{align*}
    \Pbb_{{\cal D}_n}((X_i)_{i\leq n} \notin\mathbb{D})
    &= \Pbb_{(X_i)}( \sup_{i} d(X_i, X_{(k)}(X_i)) \geq h)
    = \Pbb_{(X_i)}\paren{\cup_{i\leq n}\brace{d(X_i, X_{(k)}(X_i)) \geq h}}
    \\&\leq \sum_{i=1}^n \Pbb_{(X_i)}\paren{d(X_i, X_{(k)}(X_i)) \geq h}
    = n \Pbb_{X, {\cal D}_{n-1}}\paren{d(X, X_{(k)}(X)) \geq h}.
  \end{align*}
  This last probability has already been work out when dealing with the fully
  supervised part, and was bounded as
  \[
    \Pbb_{X, {\cal D}_{n-1}}\paren{d(X, X_{(k)}(X)) \geq h} \leq
    \exp\paren{-(n-1)p / 8}.
  \]
  as long as $k < (n-1)p / 2$. 
  Finally we have
  \[
    \sup_{X\in\X} \norm{g_n^* - g_n}_{\infty}
    \Pbb_{{\cal D}_n}((X_i)_{i\leq n} \notin\mathbb{D})
    \leq 2c_\phi n \exp\paren{-(n-1)p / 8}.
  \]
  
  For the expectation term, corresponding to datasets, ${\cal D}_n \in
  \mathbb{D}$, that cluster data accordingly to classes, we have to make
  sure that $\hat{y}_i = y_i^*$ is the only acceptable solution of
  Eq.~\eqref{eq:disambiguation}, which is true 
  as soon as the intersection of $S_j$, for $x_j$ the neighbors of $x_i$, only
  contained $y_i^*$.
  To work out the disambiguation algorithm, notice that 
  \begin{align*}
    \norm{g_n - g_n^*}_{L^1}
    &=\int_\X \norm{\sum_{i=1}^n \alpha_i(x) \phi(\hat{y}_i) - \phi(y^*_i)} \diff\nu_\X(x)
    \leq \int_\X k^{-1}\sum_{i=1}^n \ind{X_i \in{\cal N}(x)} \norm{\phi(\hat{y}_i) - \phi(y^*_i)} \diff\nu_\X(x)
    \\ &= k^{-1}\sum_{i=1}^n \Pbb_X\paren{X_i \in{\cal N}(X)} \norm{\phi(\hat{y}_i) - \phi(y^*_i)}
    \leq 2c_\phi k^{-1}\sum_{i=1}^n \Pbb_X\paren{X_i \in{\cal N}(X)} \ind{\phi(\hat{y}_i) \neq \phi(y^*_i)}.
  \end{align*}
  Finally we have, after proper conditionning, considering the variability in
  $S_i$ while fixing $X_i$ first,
  \begin{align*}
    \E_{{\cal D}_n, X}\bracket{\norm{g_n^*(X) - g_n(X)}_{\cal H}\midvert (X_i) \in \mathbb{D}}
    &= 2c_\phi k^{-1}\E_{(X_i)}\bracket{\sum_{i=1}^n \Pbb_X\paren{X_i \in{\cal
          N}(X)} \E_{(S_i)}\bracket{\ind{\phi(\hat{y}_i) \neq
          \phi(y^*_i)} \midvert (X_i)} \midvert (X_i) \in \mathbb{D}}
    \\&= 2c_\phi k^{-1}\E_{(X_i), X}\bracket{\sum_{i=1}^n \ind{X_i \in{\cal
    N}(X)} \Pbb_{(S_i)}\paren{\phi(\hat{y}_i) \neq
    \phi(y^*_i) \midvert (X_i)} \midvert (X_i) \in \mathbb{D}}.
  \end{align*}
  We design $\mathbb{D}$, because when this event holds, we know that the $k$-th
  nearest neighbor of any input $X_i$ is at distance at most $h$ of $X_i$,
  meaning the because of class separation, $y_{x_i} \in S_j$ for any $X_j \in
  {\cal N}(X_i)$. This mean that outputting $(\hat{y}_i) = (y^*_i)$ and 
  $z_j = y_j$, will lead to an optimal error in Eq.~\eqref{eq:disambiguation}. Now
  suppose that there is an other solution for Eq.~\eqref{eq:disambiguation} such
  that $\hat{y}_i \neq y_i^*$, it should also achieve an optimal error,
  therefore it should verify $z_j = \hat{y}_j$ for all $j$ as well as $\hat{y}_j
  = \hat{y}_i$ for all $j$ such that $X_j$ is one of the $k$ nearest neighbors
  of $X_i$. This implies that $\hat{y}_i \in \cap_{j;X_j \in {\cal N}(X_i)}
  S_j$, which happen with probability
  \[
    \Pbb_{(S_j)_{j;X_j\in{\cal N}(X_i)}}(\exists z \neq y_i, z \in \cap_j S_j)
    \leq m \Pbb_{S_j}(z \in S_j)^k
    \leq m \eta^k = m\exp(-k\abs{\log(\eta)}).
  \]
  With $m = \card{\Y}$ the number of element in $\Y$. We deduce that
  \[
    \Pbb_{(S_i)}\paren{\phi(\hat{y}_i) \neq \phi(y^*_i) \midvert (X_i)}
    \leq m \exp(-k\abs{\log(\eta)}).
  \]
  And because $\sum_{i=1}^n \ind{X_i \in {\cal N}(X)} = k$, we conclude that
  \[
    \E_{{\cal D}_n, X}\bracket{\norm{g_n^*(X) - g_n(X)}_{\cal H}\midvert (X_i)
      \in \mathbb{D}} \leq 2c_\phi m \exp(-k\abs{\log(\eta)}).
  \]
  Finally, adding everything together we get
  \[
    {\cal E}(f_n) \leq 8 c_\phi c_\psi \exp\paren{-\frac{np}{8}}
     + 8 c_\phi c_\psi n \exp\paren{-\frac{(n-1)p}{8}}
     + 8 c_\phi c_\psi m \exp\paren{-k\abs{\log(\eta)}}.
   \]
   as long as $k < (n-1) p /2$, which implies Theorem \ref{thm:convergence} as
   long as $n \geq 2$.
   
  \begin{remark}[Other approaches]
     While we have proceed with analysis based on local averaging methods, other
     paths could be explored to prove convergence results of the algorithm provided
     Eq.~\eqref{eq:disambiguation} and \eqref{eq:estimate}.
     For example, one could prove Wasserstein convergence of
     $\sum_{i=1}^n\delta_{(x_i, \hat{y}_i)}$ towards $\sum_{i=1}^n \delta_{(x_i, \hat{y}^*_i)}$,
     together with some continuity of the learning algorithm as a function of those
     distributions.\footnote{The Wasserstein metric is useful to think in term
       of distributions, which is natural when considering partial supervision
       that can be cast as a set of admissible fully supervised distributions.
       This approach has been successfully followed by \citet{Perchet2015} to
       deal with partial monitoring in games.}
     This analysis could be understood as tripartite:
     \begin{itemize}
     \item A disambiguation error, comparing $\hat{y}_i$ to $y_i^*$.
     \item A stability / robustness measure of the algorithm to learn $f_n$ from data
       when substituting $y_i^*$ by $\hat{y}_i$.
     \item A consistency result regarding $f_n^*$ learnt on $(x_i, y_i^*)$. 
     \end{itemize}
     Our analysis followed a similar path, yet with the first two parts tackled jointly.
   \end{remark}

\subsection{Proof of Proposition \ref{prop:init}}
\label{proof:init}

  Under the non-ambiguity hypothesis (Assumption \ref{ass:non-ambiguity}), the
  solution of Eq.~\eqref{eq:solution} is characterized pointwise by $f^*(x) = y_x$
  for all $x\in\supp\nu_\X$.
  Similarly under Assumption \ref{ass:non-ambiguity}, we have the characterization
  $f^*(x) \in \cap_{S\in\supp\nu\vert_x} S$.
  With the notation of Definition \ref{def:init}, since $f^*(x)$ minimizes
  $z\to\E_{Y\sim\mu_S}[\ell(z, Y)]$ for all $S\in\supp\nu\vert_x$, it also
  minimizes $z\to\E_{S\sim\nu\vert_x}\E_{Y\sim\mu_S} [\ell(z, Y)]$.

  For the second part of the proposition, we use the structured prediction
  framework of \citet{Ciliberto2020}.
  Define the signed measure $\mu^\circ$ defined as $\mu^\circ_\X := \nu_\X$ and
  $\mu^\circ\vert_x := \E_{S\sim\nu\vert_x}\E_{Y\sim\mu_S}[\delta_Y]$, and
  $f^\circ:\X\to\Y$ the solution $f^\circ \in
  \argmin_{f:\X\to\Y}\E_{(X, Y)\sim\mu^\circ}[\ell(f(X), Y)] = \argmin_{f:\X\to\Y}
  \E_{(X, Y)\sim\nu}\bracket{\E_{Y\sim\mu_S}[\ell(f(X), Y)]}$.
  The first part of the proposition tells us that $f^\circ = f^*$ under Assumption
  \ref{ass:non-ambiguity}.
  The framework of \citet{Ciliberto2020}, tells us that $f^\circ$ is obtained after
  decoding, Eq.~\eqref{eq:decoding}, of $g^\circ:\X\to{\cal H}$, and that if
  $g_n^\circ$ converges to $g^\circ$ with the $L^1$ norm, $f_n^\circ$ converges
  to $f^\circ$ in term of the $\mu^\circ$-risk. Under Assumption
  \ref{ass:non-ambiguity} and mild hypothesis on $\mu^\circ$, 
  it is possible to prove that convergence in term of the
  $\mu^\circ$-risk implies convergence in term of the $\mu$-risk (for example
  through calibration inequality similar to Proposition 2 of \citet{Cabannes2020}).

\subsection{Ranking with Partial ordering is a well behaved problem}
\label{proof:ranking}

  Here, we discuss about building directly $\xi_S$ to initialize our alternative
  minimization scheme or considering $\mu_S$ given by the definition of
  well-behaved problem (Definition \ref{def:init}). 
  Since the existence of $\mu_S$ implying $\xi_S$ defined as
  $\E_{Y\sim\mu_S}[\phi(Y)]$, we will only study when $\xi_S$ can be cast as a $\mu_S$.

  In ranking, we have that $\psi = -\phi$, which corresponds to ``correlation losses''.
  In this setting, we have that $\Span(\phi(\Y)) = \Span(\psi(\Y))$.
  More generally, looking at a ``minimal'' representation of $\ell$, one can
  always assume the equality of those spans, as what happens on the orthogonal of
  the intersection of those spans, does not modify the scalar product $\phi(y)^\top\psi(z)$.
  Similarly, $\xi_S$ can be restricted to $\Span(\psi(\Y))$, and
  therefore $\Span(\phi(\Y))$, which exactly the image by
  $\mu\to\E_{Y\sim\mu}[\phi(Y)]$ of the set of signed measures, showing the
  existence of a $\mu_S$ matching Definition \ref{def:init}.


\section{IQP implementation for Eq.~\eqref{eq:disambiguation}}
\label{app:diffrac}
In this section, we introduce an IQP implementation to solve for
Eq.~\eqref{eq:disambiguation}. We first mention that our alternative
minimization scheme is not restricted to well-behaved problem, before motivating
the introduction of the IQP algorithm in two different ways, and finally
describing its implementation.

\subsection{Initialization of alternative minimization for non well-behaved problem}
  Before describing the IQP implementation to solve Eq.~\eqref{eq:algo}, we would
  like to stress that, even for non well-behaved partial labelling problems, it
  is possible to search for smart ways to initialize variables of the alternative
  minimization scheme. For example, one could look at
  $z_i^{(0)} \in \cap_{j;x_j\in{\cal N}_{k_i}} S_j$, where ${\cal N}_k$ designs the
  $k$ nearest neighbors of $x_i$ in $(x_j)_{j\leq n}$, and $k_i$ is chosen such
  that this intersection is a singleton. 

\subsection{Link with Diffrac and empirical risk minimization}
  Our IQP algorithm is similar to an existing disambiguation algorithm known as
  the Diffrac algorithm \citep{Bach2007,Joulin2010}.\footnote{The Diffrac
    algorithm was first introduced for clustering, which is a classical approach
    to unsupervised learning. In practice, it consists to change the constraint set
    $C_n = \prod S_i$ by a set of the type $C_n = \argmax_{(y_i) \in \Y^n} \sum_{i,j=1}^n
    \ind{y_i\neq y_j}$ in Eqs.~\eqref{eq:disambiguation} and \eqref{eq:iqp},
    meaning that $(y_i)$ should be disambiguated into different classes.}
  This algorithm was derived by implicitly following empirical risk minimization of
  Eq.~\eqref{eq:principle}.
  This approach leads to algorithms written as
  \[
  (y_i) \in \argmin_{(y_i) \in C_n} \inf_{f\in{\cal F}} \sum_{i=1}^n \ell(f(x_i), y_i)
  + \lambda \Omega(f),
  \]
  for ${\cal F}$ a space of functions, and $\Omega:{\cal F} \to \R_+$ a measure
  of complexity. Under some conditions, it is possible to simplify the dependency
  in $f$ \citep[{\em e.g.,}][]{Xu2004,Bach2007}.
  For example, if $\ell(y, z)$ can be written as $\norm{\phi(y) - \phi(z)}^2$
  for a mapping $\phi:\X\to\Y$, {\em e.g.} the Kendall loss detailed in Section
  \ref{sec:ranking},\footnote{Since $\norm{\phi(y)}$ is constant.} and the
  search of $\phi(f):\X\to\phi(\Y)$ is relaxed as a $g:\X\to{\cal H}$. With
  $\Omega$ and ${\cal F}$ linked with kernel regression 
  on the surrogate functional space $\X\to{\cal H}$, it is possible to solve the
  minimization with respect to $g$ as $g(x_i) = \sum_{j=1}^n \alpha_j(x_i)\phi(y_i)$,
  with $\alpha$ given by kernel ridge regression \citep{Ciliberto2016}, and to
  obtain a disambiguation algorithm written as 
  \[
    \argmin_{y_i \in S_i} \sum_{i=1}^n\big\|\sum_{j=1}^n \alpha_j(x_i) \phi(y_j) - \phi(y_i)\big\|^2.
  \]
  This IQP is a special case of the one we will detail. As such, our IQP is
  a generalization of the Diffrac algorithm, and this paper provides, to our knowledge,
  {\em the first consistency result for Diffrac}.

\subsection{Link with an other determinism measure}
  While we have considered the measure of determinism given by
  Eq.~\eqref{eq:principle}, we could have considered its quadratic variant
  \[
    \mu^\star \in \argmin_{\mu\vdash\nu} \inf_{f:\X\to\Y} \E_{X\sim\nu_\X}
    \bracket{\E_{Y, Y'\sim\mu\vert_x}\bracket{\ell(Y, Y')}}.
  \]
  This correspond to the right drawing of Figure \ref{fig:objective}.
  We could arguably translate it experimentally as
  \begin{equation}
    \label{eq:iqp}
    (\hat{y}_i) \in \argmin_{(y_i) \in C_n} \sum_{i,j=1}^n \alpha_i(x_j)
    \ell(y_i, y_j),
  \end{equation}
  and still derive Theorem \ref{thm:convergence} when substituting
  Eq.~\eqref{eq:disambiguation} by Eq.~\eqref{eq:iqp}.
  When the loss is a correlation loss $\ell(y, z) = -\phi(y)^\top\phi(z)$.
  This leads to the quadratic problem
  \[
    (\hat{y}_i) \in \argmin_{(y_i) \in C_n} -\sum_{i,j=1}^n \alpha_i(x_j)
    \phi(y_i)^\top \phi(y_j).
  \]

\subsection{IQP Implementation}
  In order to make our implementation possible for any symmetric loss
  $\ell:\Y\times\Y\to\R$, on a finite space $\Y$, we introduce the following
  decomposition.

  \begin{proposition}[Quadratic decomposition]
    When $\Y$ is finite, any proper symmetric loss $\ell$ admits a decomposition
    with two mappings $\phi:\Y\to\R^m$, $\psi:\Y\to\R^m$, for a $m\in\N$ and a $c\in\R$,
    reading
    \begin{equation}
      \label{eq:qloss}
      \forall\, y, z\in \Y, \quad\ell(y, z) = \psi(y)^\top\psi(z) - \phi(y)^\top
      \phi(z)
      \qquad \text{with}\qquad \norm{\phi(y)} = \norm{\psi(y)} = c
    \end{equation}
  \end{proposition}
  \small
  \begin{proof}
    Consider $\Y = {y_1, \cdot, y_m}$ and
    $L = (\ell(y_i, y_j))_{i,j\leq m} \in \R^{m\times m}$.
    $L$ is a symmetric matrix, diagonalizable as
    \(
      L = \sum_{i=1}^m \lambda_i u_i\otimes u_i, 
    \)
    with $(u_i)$ a orthonormal basis of $\R^m$, and $\lambda_i \in \R$ its eigen values.
    We have, with $(e_i)$ the Cartesian basis of $\R^m$,
    \[
      \ell(y_j, y_k) = L_{jk} = \scap{e_j}{Le_k} = \sum_{i=1}^m (\lambda_i)_+
      \scap{e_j}{u_i} \scap{e_k}{u_i}
      - \sum_{i=1}^m (\lambda_i)_- \scap{e_j}{u_i} \scap{e_k}{u_i}.
    \]
    We build the decomposition
    \[
      \tilde{\psi}(y_k) = \paren{\sqrt{(\lambda_i)_+} \scap{e_k}{u_i}}_{i\leq m},
      \qquad\text{and}\qquad
      \tilde{\phi}(y_k) = \paren{\sqrt{(\lambda_i)_-} \scap{e_k}{u_i}}_{i\leq m}.
    \]
    It satisfies
    \(
      \ell(y_j, y_k) = \tilde\psi(y)^\top\tilde\psi(z) - \tilde\phi(y)^\top\tilde\phi(z).
    \)
    We only need to show that we can consider $\phi$ of constant norm.
    For this, first consider $C = \max_i\abs{\lambda_i}$, we have
    \(
      \norm{\tilde{\psi}(y_k)}^2 = \sum_{i=1}^m (\lambda_i)_+ \scap{u_i}{e_k}^2
      \leq C \sum_{i=1}^m \scap{u_i}{e_k}^2 = C\norm{e_k}^2 = C
    \)
    The last equalities being due to the fact that $(u_i)$ is orthonormal.
    Now, introduce the correction vector $\xi:\Y\to\R^m$,
    \(
      \xi(y_i) = \sqrt{C - \norm{\tilde\psi(y)}^2} e_i.
    \)
    And consider $\phi = \binom{\tilde\phi}{\xi}$, $\psi =
    \binom{\tilde\psi}{\xi}$. By construction, $\psi$ is of constant norm
    being equal to $C$ and that $\ell(y, z) = \psi(y)^T\psi(z) -
    \phi(y)^T\phi(z)$. Finally, because $\ell(y, z) = 0$, we also have $\phi$ of
    constant norm.
  \end{proof}
  \normalsize

  Using the decomposition Eq.~\eqref{eq:qloss}, Eq.~\eqref{eq:iqp} reads, with
  $\textbf{y} = (y_i)$
  \[
    {\bf \hat y} \in \argmin_{{\bf y}\in C_n}\sum_{i=1}^n \alpha_i(x_j) \psi(y_i) \psi(y_j) - \sum_{i=1}^n \alpha_i(x_j) \phi(y_i) \phi(y_j).
  \]
  By defining the matrix
  $A = (\alpha_i(x_j))_{ij\leq n} \in \R^{n\times n}$, 
  $\Psi({\bf y}) = (\psi(y_i))_{i\leq n} \in \R^{n\times m}$ and 
  $\Phi({\bf y}) = (\phi(y_i))_{i\leq n} \in \R^{n\times m}$, we cast it as
  \[
    {\bf \hat y} \in \argmin_{{\bf y}\in C_n}\trace\paren{A\Psi({\bf y})\Psi({\bf y})^\top} - \trace\paren{A\Phi({\bf y})\Phi({\bf y})^\top}.
  \]

  \paragraph{Objective convexification.} As $\alpha_i(x_j)$ is a measure of
  similarity between $x_i$ and $x_j$, $A$ is usually symmetric positive
  definite, making this objective convex in $\Psi$ and concave in $\Phi$.
  However, recalling Eq.~\eqref{eq:qloss}, we have
  $\trace\Phi\Phi^\top = \trace{\Psi\Psi^\top} = n c$,
  therefore considering the spectral norm of $A$, we convexify the objective as
  \[
    {\bf \hat y} \in \argmin_{{\bf y}\in C_n}\trace\paren{(\norm{A}_*I +
      A)\Psi({\bf y})\Psi({\bf y})^\top} + \trace\paren{(\norm{A}_*I - A)\Phi({\bf y})\Phi({\bf y})^\top}.
  \]
  Considering
  \[
    B = \paren{\begin{array}{cc} \norm{A}_*I + A & 0 \\ 0 & \norm{A}_*I -  A \end{array}}
  \qquad\text{and}\quad
  \Xi({\bf y}) = \binom{\Psi({\bf y})}{\Phi({\bf y})},
  \]
  allow to simplify this objective as
  \[
    {\bf \hat y} \in \argmin_{{\bf y}\in C_n}\trace\paren{B\Xi({\bf y})\Xi({\bf y})^\top}.
  \]
  When parametrized by $\xi = \Xi({\bf y})$, this is an optimization problem with
  a convex quadratic objective and ``integer-like'' constraint $\xi \in
  \Xi(C_n)$, identifying to an integer quadratic program (IQP).

  \paragraph{Relaxation.} IQP are known to be NP-hard, several tools exists in
  literature and optimization library implementing them. The most classical
  approach consists in relaxing the integer constraint $\xi \in \Xi(C_n)$ into
  the convex constraint $\xi \in \hull(\Xi(C_n))$, solving the resulting convex
  quadratic program, and projecting back the solution towards an extreme of the
  convex set. Arguably, our alternative minimization approach is a better
  grounded heuristic to solve our specific disambiguation problem.

\section{Example with graphical illustrations}
\label{app:example}

\begin{figure*}[t]
  \centering
  \begin{tikzpicture}[scale=3]
  \coordinate(a) at (0, 0);
  \coordinate(b) at ({1/2}, {sin(60)});
  \coordinate(c) at (1, 0);
  \coordinate(ha) at ({3/4}, {sin(60)/2});
  \coordinate(hb) at ({1/2}, 0);
  \coordinate(hc) at ({1/4}, {sin(60)/2});

  \coordinate (mc) at (1,-.25);
  \coordinate (mca) at (0,-.25);
  \fill[fill=white] (a) -- (c) -- (mc) -- (mca) -- cycle;

  \fill[fill=red!20] (a) -- (hb) -- (ha) -- (hc) -- cycle;
  \fill[fill=green!20] (b) -- (ha) -- (hc) -- cycle;
  \fill[fill=blue!20] (c) -- (ha) -- (hb) -- cycle;
  \draw (a) node[anchor=north east]{a} -- (b) node[anchor=south]{b} --
  (c) node[anchor=north west]{c} -- cycle;
  \draw (hb) -- (ha) -- (hc);
  \node at ({3/8}, {sin(60)/4}) {$R_a$};
  \node at ({1/2}, {3*sin(60)/4 - 1/16}) {$R_b$};
  \node at ({3/4}, {sin(60)/4 - 1/16}) {$R_c$};
\end{tikzpicture}
\begin{tikzpicture}[scale=3]
  \coordinate(a) at (0, 0);
  \coordinate(b) at ({cos(60)}, {sin(60)});
  \coordinate(c) at (1, 0);
  \coordinate(r) at ({1/4}, 0);
  \coordinate(s) at ({7/32 + 1/8}, {7*sin(60)/16});
  \coordinate(t) at ({3/8 + 1/4}, {3*sin(60)/4});

  \coordinate (mc) at (1,-.25);
  \coordinate (mca) at (0,-.25);
  \fill[fill=white] (a) -- (c) -- (mc) -- (mca) -- cycle;

  \fill[fill=black!20] (c) -- (r) -- (s) -- (t) -- cycle;
  \fill[fill=black!10] (a) -- (r) -- (s) -- (t) -- (b) -- cycle;
  \draw (a) node[anchor=north east]{a} -- (b) node[anchor=south]{b} --
  (c) node[anchor=north west]{c} -- cycle;
  \draw (c) -- (r) -- (s) -- (t) -- cycle;
  \node at ({5/8}, {sin(60)/3 - 1/16}) {$R_{\nu}$};
\end{tikzpicture}
\begin{tikzpicture}[scale=3]  
  \coordinate (a) at (0,0) ;
  \coordinate (b) at ({1/2},{sin(60)}) ;
  \coordinate (c) at (1,0);
  \coordinate (mb) at (.25, {sin(60)});
  \coordinate (mba) at ({-.25*cos(30)},{.25*sin(30)});
  \coordinate (mc) at (1,-.25);
  \coordinate (mca) at (0,-.25);
  \coordinate(ha) at ({3/4}, {sin(60)/2});
  \coordinate(hb) at ({1/2}, 0);
  \coordinate(hc) at ({1/4}, {sin(60)/2});

  \foreach \x in {0,.05,...,.25}
  \draw[gray, rotate=30] ({sin(60)}, \x) -- ({sin(60) - (tan(60)*\x)}, \x) -- ({sin(60) - (tan(60)*\x)}, -\x) -- ({sin(60)}, {-\x}); 
  \foreach \x in {0.25,.3,...,.5}
  \draw[gray, rotate=30] ({sin(60)}, \x) -- ({sin(60)-tan(60)*(.5-\x)}, \x);
  \foreach \x in {0.25,.3,...,.5}
  \draw[gray, rotate=30] ({sin(60)-tan(60)*(.5-\x)}, -\x) -- ({sin(60)}, {-\x}); 
  \foreach \x in {.25,.3,...,.5}
  \draw[gray, rotate=30] ({sin(60) - (tan(60)*\x)}, .25) -- ({sin(60) - (tan(60)*\x)}, -.25); 

  \fill[fill=white] (a) -- (c) -- (mc) -- (mca) -- cycle;
  \fill[fill=white] (a) -- (b) -- (mb) -- (mba) -- cycle;
  \draw (a) node[anchor=north east]{a} -- (b) node[anchor=south]{b} -- (c) node[anchor=north west]{c} -- cycle;
  \draw[dotted, thick] (hc) -- (ha) -- (hb);
\end{tikzpicture}
\begin{tikzpicture}[scale=3]
  \coordinate (a) at (0,0) ;
  \coordinate (b) at ({1/2},{sin(60)}) ;
  \coordinate (c) at (1,0) ;
  \coordinate (mb) at (.25, {sin(60)});
  \coordinate (mba) at ({-.25*cos(30)},{.25*sin(30)});
  \coordinate (mc) at (1,-.25);
  \coordinate (mca) at (0,-.25);
  \coordinate(ha) at ({3/4}, {sin(60)/2});
  \coordinate(hb) at ({1/2}, 0);
  \coordinate(hc) at ({1/4}, {sin(60)/2});

  \foreach \x in {0,.05,...,.5}
  \draw[gray,rotate=30] ({sin(60)}, 0) + (0, {\x}) arc (90:270:{sqrt(3)*\x} and {\x});

  \fill[fill=white] (a) -- (c) -- (mc) -- (mca) -- cycle;
  \fill[fill=white] (a) -- (b) -- (mb) -- (mba) -- cycle;
  \draw (a) node[anchor=north east]{a} -- (b) node[anchor=south]{b} -- (c) node[anchor=north west]{c} -- cycle;
  \draw[dotted, thick] (hc) -- (ha) -- (hb);
\end{tikzpicture}
  \caption{Exposition of a pointwise problem in the simplex $\prob\Y$, with
    $\Y = \brace{a, b, c}$ and a proper symmetric loss defined by $\ell(a, b) =
    \ell(a, c) = \ell(b, c) / 2$.
    (Left) Representation of the decision regions
    $R_z = \brace{\mu\in\prob{\Y}\midvert z\in \argmin_{z'\in\Y}
      \E_{Y\sim\mu}[\ell(z, y)]}.$ for $z\in\Y$.
    (Middle Left) Representation of
    $R_\nu = \brace{\mu\in\prob{\Y} \midvert \mu\vdash\nu}$ for
    $\nu = (5\delta_{\brace{a, b, c}} + \delta_{\brace{c}} + 
    \delta_{\brace{a, c}} + \delta_{\brace{b, c}}) / 8$
    (Middle Right) Level curves of the piecewise function
    $\prob{\Y}\to\R; \mu\to\min_{z\in\Y}\E_{Y\sim\mu}[\ell(z, Y)]$
    corresponding to Eq.~\eqref{eq:principle}.
    (Right) Level curves of the quadratic function
    $\prob{\Y}\to\R; \mu\to \E_{Y, Y'\sim\mu}[\ell(Y, Y')]$.
    Our disambiguation \eqref{eq:principle} corresponds to minimizing the
    concave function represented on the middle right drawing on the convex domain
    represented on the middle left drawing.}  
  \label{fig:objective}
\end{figure*}
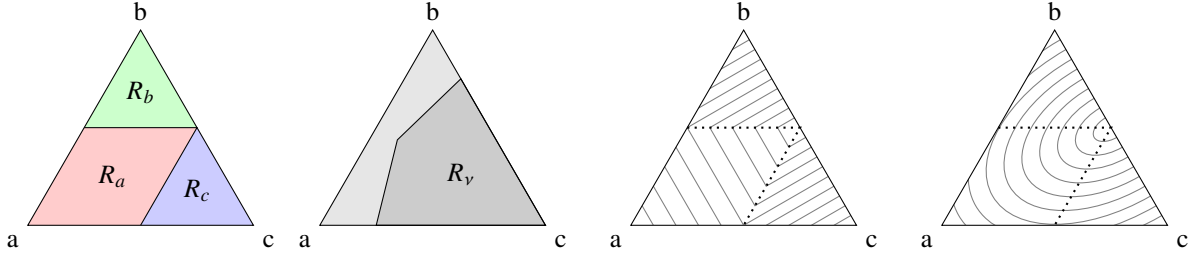

To ease the understanding of the disambiguation principle \eqref{eq:principle},
we provide a toy example with a graphical illustration, Figure~\ref{fig:objective}.
Since Eq.~\eqref{eq:principle} decorrelates inputs, we will consider $\X$ to be
a singleton, in order to remove the dependency to $\X$.
In the following, we consider $\Y = \brace{a, b, c}$, with the loss given by
\[
  L = (\ell(y, z))_{y, z\in \Y} =
  \paren{\begin{array}{ccc}
    0 & 1 & 1 \\
    1 & 0 & 2 \\
    1 & 2 & 0
  \end{array}}.
\]

This problem can be represented on a triangle through the embedding of probability
measures reading 
$\xi:\prob{\Y}\to \R^3; \mu\to\mu(a)e_1 + \mu(b)e_2 + \mu(c)e_3$,
and onto the triangle $\brace{z\in\R_+^3 \midvert z^\top 1 = 1}$.
Note that $\xi$ can be extended from any signed measure of
total mass normalized to one onto the plane $\brace{z\in\R^3 \midvert z^\top 1 = 1}$, as
well as the drawings Figure \ref{fig:objective} can be extended onto the affine
span of the represented triangles. 
The objective \eqref{eq:principle} reads pointwise as
$\prob{\Y}\to\R; \mu \to \min_{i\leq 3} e_i^\top L \xi(\mu)$, while its
quadratic version reads $\prob{\Y}\to\Y; \mu\to \xi(\mu)^\top L \xi(\mu)$.
Note that while $L$ is not definite negative, one can check that the restriction
of $\R^3 \to \R; z \to z^\top Lz$ to the definition domain
$\brace{z\in\R^3\midvert z^\top 1 = 1}$ is concave, as suggested by the right
drawing of Figure \ref{fig:objective}.

It should be noted that $(\ell, \nu)$ being a well-behaved partial labelling problem
can be understood graphically, as having the intersection of the decision regions
$\cap_{z\in S} R_z$ non-empty for any set $S$ in the support of $\nu$.
As such, it is easy to see that our toy problem is well-behaved for any
distribution $\nu$.
Formally, to match Definition \ref{def:init}, we can define $\mu_{\brace{e}} = \delta_e$ for $e\in\brace{a, b, c}$ and 
\[
  \mu_{\brace{a,b}} = .5\delta_a + .5\delta_b,\quad
  \mu_{\brace{a,c}} = .5\delta_b + .5\delta_c,\quad
  \mu_{\brace{b, c}} = \delta_b + \delta_c - \delta_a,\quad
  \mu_{\brace{a, b, c}} = .5\delta_b + .5\delta_c.
\]
Graphically $\xi(\mu_{\brace{a,b}})$ can be chosen as any points on the
horizontal dashed line on the middle right drawing of Figure \ref{fig:objective}
(similarly for $\xi\mu_{\brace{a, c}}$), while $\xi(\mu_{\brace{a, b, c}})$ has to
be chosen has the intersection $.5e_2 + .5e_3$, and while $\xi(\mu_{\brace{b,c}})$
has to be chosen outside the simplex on the half-line leaving $.5e_2 + .5e_3$
supported by the perpendicular bisector of $[e_2, e_3]$ and not containing $e_1$.


\section{Experiments}
\label{app:experiments}
While our results are much more theoretical than experimental, out of principle,
as well as for reproducibility, comparison and usage sake, we detail our experiments. 

\subsection{Interval regression - Figure \ref{fig:ir}}
Figure \ref{fig:ir} corresponds to the regression setup consisting of learning
$f^*:[0, 1]\to\R; x\to\sin(\omega x)$, with $\omega=10\approx 3\pi$.
The dataset represented on Figure \ref{fig:ir} is collected in
the following way. We sample $(x_i)_{i\leq n}$ with $n = 10$, uniformly at
random on $\X=[0, 1]$, after fixing a random seed for reproducibility.
We collect $y_i = f(x_i)$. We create $(s_i)$ by sampling $u_i$ uniformly on
$[0,1]$, defining $r_i = r - \gamma \log(u_i)$, with $r=1$ and $\gamma =
3^{-1}$, sampling $c_i$ uniformly at random on $[0, r_i]$, and defining
$s_i = y_i + \sign(y_i)\cdot c_i + [-r_i, r_i]$.
The corruption is skewed on purpose to showcase disambiguation instability of
the baseline \eqref{eq:baseline} compared to our method.
We solve Eq.~\eqref{eq:disambiguation} with alternative minimization,
initialized by taking $y_i^{(0)}$ at the center of $s_i$, and stopping the
minimization scheme when $\sum_{i\leq n}\vert y_i^{(t+1)} - y_i^{(t)}\vert < \epsilon$
for $\epsilon$ a stopping criterion fixed to $10^{-6}$.
For $x\in \X$, the inference Eqs.~\eqref{eq:estimate} and \eqref{eq:baseline} is
done through grid search, considering, for $f_n(x)$, 1000 guesses dividing
uniformly $[-6, 6] \subset \Y = \R$.
We consider weights $\alpha$ given by kernel ridge regression with Gaussian
kernel, defined as
\[
  \alpha(x) = (K+n\lambda I)^{-1}K_x \in \R^n, \quad
  K = (k(x_i, x_j))_{i,j\leq n} \in \R^{n\times n},\quad
  K_x = (k(x_i, x))_{i\leq n} \in \R^n,\quad
  k(x, x') = \exp\paren{-\frac{\norm{x-x'}^2}{2\sigma^2}},
\]
with $\lambda$ a regularization parameter, and $\sigma$ a standard deviation parameter.
In our simulation, we fix $\sigma = .1$ based on simple considerations on the
data, while we consider $\lambda \in [10^{-1}, 10^{-3}, 10^{-6}]$.
The evaluation of the mean square error between $f_n$ and $f^*$, which is
equivalent to evaluating the risk with the regression loss $\ell(y, z) = \norm{y - z}^2$,
is done by considering 200 points dividing uniformly $\X = [0,1]$ and evaluating
$f_n$ and $f^*$ on it.
The best hyperparameter $\lambda$ is chosen by minimizing this error.
It leads to $\lambda = 10^{-1}$ for the baseline \eqref{eq:baseline}, and
$\lambda = 10^{-6}$ for our algorithm \eqref{eq:disambiguation} and
\eqref{eq:estimate}.
This difference in $\lambda$ is normal since both methods are not estimating
the same surrogate quantities.
The fact that $\lambda$ is smaller for our algorithm is natural as our
disambiguation objective \eqref{eq:disambiguation} already has a regularization
effect on the solution.\footnote{Moreover, the analysis in \citet{Cabannes2020}
  suggests that the baseline is estimating a surrogate function in $\X\to 2^\R$,
  while our method is estimating a function in $\X\to\R$, which is a much
  smaller function space, hence needing less regularization. However, those
  reflections are based on upper bounds, that might be sub-optimal, which
  could invalidate those considerations.}
Note that we used the same weights $\alpha$ for Eq.~\eqref{eq:disambiguation}
and Eq.~\eqref{eq:estimate}, which is suboptimal, but fair to the baseline, as,
consequently, both methods have the same number of hyperparameters.

\begin{figure*}[t]
  \centering
  \vskip -0.1in
  \includegraphics{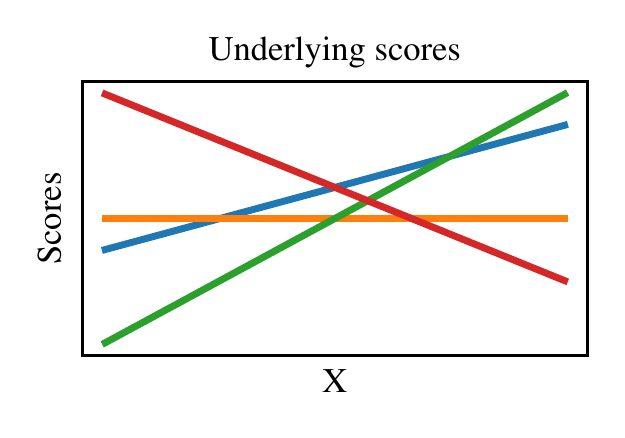}
  \includegraphics{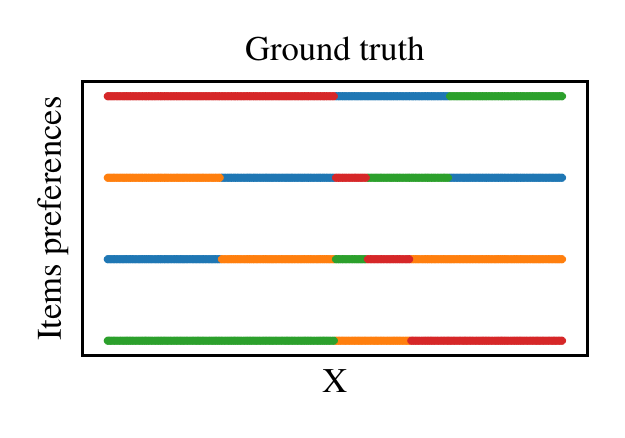}
  \vskip -0.2in
  \caption{
    Ranking setting.
    We consider $\X$ an interval of $\R$, and $\Y = \Sfrak_m$
    with $m=4$ on the figure. (Right) To create a ranking dataset, we sample randomly
    $m$ lines in $\R^2$, embedding a value, or equivalently a score, associated
    to each items as a function of the input $x$.
    (Left) By ordering those lines, we create preferences
    between items as a function of $x$. On the figure, when $x$ is small, the
    ``red'' item is prefered over the ``orange'' item, itself prefered over the ``blue''
    item, itself prefered over the ``green'' item. While when $x$ is big, ``green'' is
    prefered over ``blue'', prefered over ``orange'', prefered over ``red''.
    We create a partial labelling dataset by sampling $(x_i) \in \X^n$, and
    providing only partial ordering that the $(y_i)$ follow. For example, for a
    small $x$, we might only give the partial information that ``red'' is
    prefered over ``blue''.
  }
  \label{fig:rk_set}
  \vskip -0.1in
\end{figure*}

\subsection{Classification - Figure \ref{fig:cl}}
Figure \ref{fig:cl} corresponds to classification problems, based on real
dataset from the LIBSVM datasets repository. At the time of writing,
the datasets are available at
\url{https://www.csie.ntu.edu.tw/~cjlin/libsvmtools/datasets/multiclass.html}.
We present results on the ``Dna'' and ``Svmguide2'' datasets, that both have 3
classes ($m=3$), and respectively have 4000 samples with 180 features ($n=4000$,$d=180$)
and 391 samples with 20 features ($n = 391$, $d=20$). 

In term of {\em complexity}, when $\Y = \bbracket{1, m} = \brace{1, 2, \cdots, m}$,
and weights based on kernel ridge regression with Gaussian kernel as described in
the last paragraph the complexity of performing inference for Eqs.~\eqref{eq:estimate}
and \eqref{eq:baseline} can be done in $O(nm)$ in time and $O(n+m)$ in space,
where $n$ is the number of training samples \citep{Nowak2019,Cabannes2020}.
The disambiguation \eqref{eq:disambiguation} performed with alternative
minimization is done in $O(cn^2m)$ in time and in $O(n(n+m))$ in space,
with $c$ the number of steps in the alternative minimization scheme. In practice,
$c$ is really small, which can be understood since we are minimizing a concave
function and each step leads to a guess on the border of the constraint domain.

Based on the dataset $(x_i, y_i)$, we create $(s_i)$ by sampling it accordingly to
$\gamma \delta_{\brace{y_i}} + 1-\gamma \delta_{\brace{y, y_i}}$, with $y$ the
most present labels in the dataset (indeed we choose the two datasets because
they were not too big and presenting unequal labels proportion), and $\gamma\in[0,1]$
the corruption parameter represented in percentage on the $x$-axis of Figure
\ref{fig:cl}. This skewed corruption allows to distinguish methods and invalidate
the simple approach consisting to averaging candidate (AC) in set to recover $y_i$
from $s_i$, which works well when data are {\em missing at random} \citep{Heitjan1991}.
We  separate $(x_i, s_i)$ in 8 folds, consider $\sigma \in
d\cdot[1, .1, .01]$, where $d$ is the dimension of $\X$, and
$\lambda \in n^{-1/2}\cdot [1, 10^{-3}, 10^{-6}]$, where $n$ is the number of
data. We test the different hyperparameter setup and reported the best error for
each corruption parameter on Figure \ref{fig:cl}.
Those errors are measured with the 0-1 loss,
computed as averaged over the 8 folds, {\em i.e.} cross-validated, which standard
deviation represented as errorbars on the figure.
The best hyperparameter generally corresponds to $\sigma = .1$ and $\lambda =
10^{-3}$ when the corruption is small and $\sigma = 1$, $\lambda = 10^{-3}$ when
the corruption is big. 
Differences between cross-validated error and testing error were small, and we
presented the first one out of simplicity.

In term of {\em energy cost}, the experiments were run on a personal laptop that
has two processors, each of them running 2.3 billion instructions per second. 
During experiments, all the data were stored on the random access memory of 8GB.
Experiments were run on Python, extensively relying on the NumPy library
\citep{numpy}.
The heaviest computation is Figure \ref{fig:cl}. 
Its total runtime, cross-validation included, was around 70 seconds.
This paper is the results of experimentations, 
we evaluate the total cost of our experimentations to be three orders of
magnitude higher than the cost of reproducing the final computations presented
on Figure \ref{fig:ir}, \ref{fig:cl} and \ref{fig:ss}. 
The total computational energy cost is very negligible.

\subsection{Semi-supervised learning - Figure \ref{fig:ss}}

On Figure \ref{fig:ss}, we review a semi-supervised
classification problem with $\Y = \bbracket{1, 4}$, $\X = [-4.5, 4.5]^2$,
$\mu_\X$ only charging $\brace{x=(x_1, x_2)\in\R^2\midvert x_1^2 + x_2^2 \in \N^*}$
and the solution $f^*:\X \to \Y$ being defined almost everywhere as
$f^*(x) = x_1^2 + x_2^2$.
We collect a dataset $(x_i, s_i)$, by sampling 2000 points $\theta_i$ uniformly at
random on $[0, 1]$, as well as $r_i$ uniformly at random in
$\bbracket{1, 4} = \brace{1, 2, 3, 4}$,
before building $x_i = r_i\cdot (\cos(2\pi \theta_i), \sin(2\pi\theta_i)) \in \X$,
and $s_i = \Y$.
We add four labelled points to this dataset
$x_{2001} = (-2\sqrt{3}, 2)$ with $s_{2001} = \brace{4}$,
$x_{2001} = (1, -2\sqrt{2})$ with $s_{2002} = \brace{3}$,
$x_{2001} = (-\sqrt{3}, -1)$ with $s_{2003} = \brace{2}$ and
$x_{2001} = (-1, 0)$ with $s_{2004} = \brace{1}$.
We designed the weights $\alpha$ in Eq.~\eqref{eq:disambiguation} with $k$-nearest
neighbors, with $k=20$, and solve this equation with a variant of alternative
minimization, leading to the optimal solution $\tilde{y}_i = y_i^*$.
In order to be able to compute the baseline \eqref{eq:baseline}, we design
weights $\alpha$ for the inference task based on Nadaraya-Watson estimators with
Gaussian kernel, defined as $\alpha_i(x) = \exp\paren{\norm{x-x_i}^{2} / h}$,
with $h = .08$. We solve the inference task on a grid of $\X$ composed of 2500
points, and artificially recreate the observation to make them neat and reduce the
resulting pdf size.
Note that it is possible to design weights $\alpha$ that capture the cluster
structure of the data, which, in this case, will lead to a nice behavior of the
baseline as well as our algorithm.
Arguably, this experiment showcase a regularization property of our algorithm
\eqref{eq:disambiguation}.

\subsection{Ranking with partial ordering}

To conclude this experiment section, we look at ranking with partial ordering.
We refer to Section \ref{sec:ranking} for a clear description of this instance
of partial labelling.
We provide to the reader eager to use our method, an implementation of our
algorithm, available online at \url{https://github.com/VivienCabannes/partial_labelling}.
It is based on LP relaxation of the NP-hard minimum feedback arcset problem.
This relaxation was proven exact when $m \leq 6$ by \citet{Cabannes2020}.
The LP implementation relies on CPLEX \citep{Cplex}.
As complementary experiments, we will not provide much reproducibility details, those
details would be really similar to the previous paragraphs, and the curious
reader could run our code instead.
We present our ranking setup on Figure \ref{fig:rk_set} and our results on Figure \ref{fig:rk}.

\begin{figure*}[t]
  \centering
  \vskip -0.1in
  \includegraphics{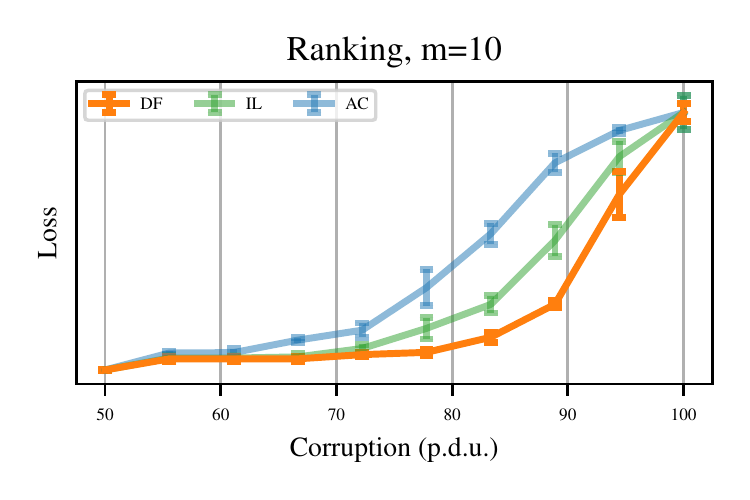}
  \vskip -0.2in
  \caption{
    Performance of our algorithm for ranking with partial ordering.
    This figure is similar to Figure
    \ref{fig:cl}, but is based on the ranking problem illustrated on Figure
    \ref{fig:rk_set}. For this figure, we consider $m = 10$, as it is
    arguably the limit where the LP relaxation provided by \citet{Cabannes2020}
    of the NP-hard minimum feedback arcset problem still performs well.
    The corruption parameter corresponds to the proportion of coordinates lost
    in the Kendall embedding when creating $s_i$ from $y_i$.
    Because the Kendall embedding satisfies transitivity constraints, a corruption
    smaller than 50\% is almost ineffective to remove any information.
    In this figure, we observe a similar behavior for ranking to the one
    observed for classification on Figure \ref{fig:cl},
    suggesting that those empirical findings are not spurious.
  }
  \label{fig:rk}
  \vskip -0.1in
\end{figure*}


\end{document}